\documentclass{tlp}
\usepackage{aopmath}

\usepackage[pdftex]{graphicx}
\usepackage{algorithmic}
\usepackage{algorithm}
\usepackage{color}
\usepackage{times,helvet,courier}
\usepackage{latexsym,amssymb}
\usepackage{xspace,epsfig,url}
\usepackage{amssymb}

\newcommand{\nop}[1]{}

\newtheorem{example}{Example}

\def\ii#1{\hbox{\it #1\/}}
\def\is#1{\hbox{\scriptsize\it #1\/}}
\def\no{\ii{not}\ }
\def\lar{\leftarrow}
\def\ba{\begin{array}}
\def\ea{\end{array}}
\def\beq{\begin{equation}}
\def\eeq#1{\label{#1}\end{equation}}
\def\dom{\textrm{dom}\ }

\newcommand{\resp}{resp.\xspace}

\def\clasp{{\sc clasp}}

\def\claspnk{{\sc clasp-nk}}
\def\gringo{{\sc gringo}}

\newcommand{\sol}[1]{\ensuremath{\mathit{Sol}(#1)}}

\newcommand{\asp}{\ensuremath{\mathcal{AS(P)}}}

\newcommand{\Pol}{\ensuremath{\rm{P}}}

\newcommand{\NP}{\ensuremath{\rm{NP}}}

\newcommand{\FPNP}{\ensuremath{\rm{FP^{NP}}}}
\newcommand{\FNP}{\ensuremath{\rm{FNP}}}
\newcommand{\FNPLog}{\ensuremath{\rm{FNP/\!/\!\log}}}

\begin{document}

\bibliographystyle{acmtrans}

\title[Finding Similar/Diverse Solutions in Answer Set
Programming]{Finding Similar/Diverse Solutions \\ in Answer Set
Programming%
\thanks{Part of the results in this paper are contained,
in preliminary form, in {\em Proceedings of the 25'th International
Conference on Logic Programming (ICLP 2009)}. This work was
partially supported by FWF (Austrian Science Funds) project
\mbox{P20841}, the Wolfgang Pauli Institute, and TUBITAK Grants
107E229 and 108E229.} }

\author[Eiter et. al.]
{THOMAS EITER\\
Institute of Information Systems, Vienna University of Technology,
Vienna, Austria\\
E-mail: eiter@kr.tuwien.ac.at\\
\and ESRA ERDEM and HALIT ERDOGAN\\
Faculty of Engineering and Natural Sciences, Sabanci University, Istanbul, Turkey\\
E-mail: \{esraerdem,halit\}@sabanciuniv.edu\\
\and MICHAEL FINK\\
Institute of Information Systems, Vienna University of Technology, Vienna, Austria\\
E-mail: fink@kr.tuwien.ac.at}

 \submitted{11 January 2011}
 \revised{27 July 2011}
 \accepted{11 August 2011}


\maketitle

\label{firstpage}

\begin{abstract}
For some computational problems (e.g., product configuration,
planning, diagnosis, query answering, phylogeny reconstruction)
computing a set of similar/diverse solutions may be desirable for
better decision-making. With this motivation, we have studied
several decision/optimization versions of this problem in the
context of Answer Set Programming (ASP), analyzed their
computational complexity, and introduced offline/online methods to
compute similar/diverse solutions of such computational problems
with respect to a given distance function. All these methods rely on
the idea of computing solutions to a problem by means of finding the
answer sets for an ASP program that describes the problem. The
offline methods compute all solutions of a problem in advance using
the ASP formulation of the problem with an existing ASP solver, like
\clasp, and then identify similar/diverse solutions using some
clustering methods (possibly in ASP as well). The online methods
compute similar/diverse solutions of a problem following one of the
three approaches: by reformulating the ASP representation of the
problem to compute similar/diverse solutions at once using an
existing ASP solver; by computing similar/diverse solutions
iteratively (one after other) using an existing ASP solver; by
modifying the search algorithm of an ASP solver to compute
similar/diverse solutions incrementally. All these methods are
sound; the offline method and the first online method are complete
whereas the others are not. We have modified \clasp\ to implement
the last online method and called it \claspnk. In the first two
online methods, the given distance function is represented in ASP;
in the last one however it is implemented in C++. We have showed the
applicability and the effectiveness of these methods using \clasp\
or \claspnk\ on two sorts of problems with different distance
measures: on a real-world problem in phylogenetics (i.e.,
reconstruction of similar/diverse phylogenies for Indo-European
languages), and on several planning problems in a well-known domain
(i.e., Blocks World). We have observed that in terms of
computational efficiency (both time and space) the last online
method outperforms the others; also it allows us to compute
similar/diverse solutions when the distance function cannot be
represented in ASP (e.g., due to some mathematical functions not
supported by the ASP solvers) but can be easily implemented in C++.
\end{abstract}

\section{Introduction}\label{sec_Intr}

For many computational problems, the main concern is to find a best
solution (e.g., a most preferred product configuration, a shortest
plan, a most parsimonious phylogeny) with respect to some
well-described criterion. On the other hand, in many real-world
applications, computing a subset of good solutions that are
similar/diverse may be desirable for better decision-making. For one
reason, the given computational problem may have too many good
solutions, and the user may want to examine only a few of them to
pick one. Also, in many real-world applications the users usually
take into account furthermore criterion that are not included in the
formulation of the optimization problem; in such cases, good
solutions similar to a best one may also be useful. Here are some
examples from several domains illustrating the usefulness of finding
similar/diverse solutions.

\paragraph{\bf Product configuration}
Consider, for instance, a variation of the example given in
\cite{heb05} about buying a car. Suppose that there is a product
advisor that asks customers about their constraints/preferences
about a car, and then lists the available ones that match their
constraints/preferences. However, such a list may be too long. In
that case, the customer might ask for a few cars that not only suit
her constraints/preferences but also are as diverse as possible.
Then, if she likes one particular car among them, she might ask for
other cars that are as similar as possible to this particular car.
Also, the customer may have other (possibly secondary) criterion
that the product advisor has not asked about; and thus the best
alternatives listed by the product advisor may not cover some of the
good possibilities. Then, the user may ask for a couple of good
enough configurations that are distant from a set of best
configurations.

\paragraph{\bf Planning}
Given an initial state, goal conditions, and a description of
actions, planning is the problem of finding a sequence of actions
(i.e., a plan) that would lead the initial state to a goal state.
Planning is applied in various domains, such as robotics, web
service composition, and genome rearrangement. In planning, it may
be desirable to compute a set of plans that are similar to each
other, so that, when the plan that is being executed fails, one can
switch to a most similar one. For instance, consider a variation of
the example given in~\cite{SrivastavaNguyenEtAl-Domain-2007} in
connection with modeling web service composition as a planning
problem~\cite{McIlraithS02}: suppose that the web service engine
computes a plan/composition; then it can compute a set of
compositions similar to this particular one, so that if a failure
occurs while executing one composition, an alternative composition
which is less likely to be failing simultaneously can be
used~\cite{ChafleDKMS06}.
Alternatively, let us consider planning in the context of robotics
in a dynamic environment with uncertainties. If the plan failure is,
for instance,  due to some collisions with an obstacle as in the
scenarios presented in \cite{cal09}, the agent may want to find a
plan that is distant from the previously computed plan so that it
does not collide with the obstacle again.

\paragraph{\bf Phylogeny reconstruction}
Phylogeny reconstruction is the problem of inferring a leaf-labeled
rooted directed tree (i.e., phylogeny) that would characterize the
evolutionary relations between a family of species based on their
shared traits. Phylogeny reconstruction is important for research
areas as disparate as genetics, historical linguistics, zoology,
anthropology, archaeology, etc.. For example, a phylogeny of
parasites may help zoologists to understand the evolution of human
diseases~\cite{PhyEcologyBehave-91}; a phylogeny of languages may
help scientists to better understand human
migrations~\cite{PrehistoryOfAustralia-82}. For a given set of
taxonomic units, some existing phylogenetic systems, like that of
\cite{bro05,bro07}, generate more than one good phylogeny that
explains the evolutionary relationships between the given taxonomic
units. However, usually there are too many phylogenies computed by a
system, an expert needs to compare these phylogenies in detail, by
analyzing the similar/diverse ones with respect to some distance
measure, to pick the most plausible ones.

\bigskip

Motivated by such examples, we have studied various problems related
to computing similar/diverse solutions in the context of a new
declarative programming paradigm, called Answer Set Programming
(ASP)~\cite{Lif-What-08}. We have also introduced general
offline/online methods in ASP that can be applied to various domains
for such computations.

In ASP, a combinatorial search problem is represented as an ASP
program whose models (called ``answer sets'') correspond to the
solutions. The answer sets for the given formalism can be computed
by special systems called answer set solvers, such as
\clasp~\cite{GebKauNeu-clasp-07}. Due to the expressive formalism of
ASP that allows us to represent, e.g., negation, defaults,
aggregates, recursive definitions, and due to the continuous
improvements of efficiency of solvers, ASP has been used in  a
wide-range of knowledge-intensive applications from different
fields, such as product configuration~\cite{SoiNie-Developing-98},
planning \cite{Lif-Action-99}, phylogeny
reconstruction~\cite{bro07}, developing a decision support system
for a space shuttle~\cite{NogBalGel-AProlog-01}, multi-agent
planning~\cite{SonPonSak-Logic-09}, answering biomedical
queries~\cite{bod08}. For many of these applications, finding
similar/diverse solutions (and thus the methods we have developed
for computing similar/diverse solutions in ASP) could be useful.

The main contributions of this paper can be summarized as follows.

\begin{itemize}

\item We have described mainly two kinds of computational problems related to finding
similar/diverse solutions of a given problem, in the context of ASP
(Section~\ref{sec_CompProb}). Both kinds of problems take as input
an ASP program $P$ that describes a problem, a distance measure
$\Delta$ that maps a set of solutions of the problem to a
nonnegative integer, and two nonnegative integers $n$ and $k$. One
problem asks for a set $S$ of size $n$ that contains $k$-similar
(\resp $k$-diverse) solutions, i.e., $\Delta(S) \leq k$ (\resp
$\Delta(S) \geq k$); the other problem asks, given a set $S$ of $n$
solutions, for a $k$-close (\resp $k$-distant) solution $s$ (\resp
$s\not\in S$), i.e., $\Delta(S \cup\{s\}) \leq k$ (\resp $\Delta(S
\cup \{s\}) \geq k$). Note that, by fixing some parameters and
minimizing/maximizing others, we can turn them into various related
optimization problems.

\item We have studied the computational complexity of these decision/optimization
problems establishing completeness results under reasonable
assumptions for the problem parameters (Section~\ref{sec_CompResu}).

\item We have introduced an offline method to compute a set of $n$
$k$-similar (\resp $k$-diverse) solutions to a given problem, by
computing all solutions in advance using ASP and then finding
similar (\resp diverse) solutions using some clustering methods,
possibly in ASP as well (Section~\ref{ssec_OfflMeth}). This method
is sound and complete, assuming that the ASP formulations are
correct.

\item We have introduced three online methods to compute a set of $n$
$k$-similar (\resp $k$-diverse) solutions to a given problem
(Sections~\ref{ssec_OnliMeth1},~\ref{ssec_OnliMeth2}
and~\ref{ssec_OnliMeth3}).

\begin{itemize}
\item
Online Method 1 reformulates the given program to compute
$n$-distinct solutions and formulates the distance function as an
ASP program, so that all $n$ $k$-similar (\resp $k$-diverse)
solutions can be extracted from an answer set for the union of these
ASP programs.

\item
Online Method 2 does not modify the ASP program encoding the
problem, but formulates the distance function as an ASP program, so
that a unique $k$-close (\resp $k$-distant) solution can be
extracted from an answer set for the union of these ASP programs and
a previously computed solution; by iteratively computing $k$-close
(\resp $k$-distant) solutions one after other, we can compute online
a set of $n$ $k$-similar (or $k$-diverse) solutions.

\item
Online Method 3 does not modify the ASP encoding of the problem, and
does not formulate the distance function as an ASP program, but it
modifies the search algorithm of an ASP solver, in our case {\sc
Clasp} \cite{CLASPGebserSchaub-07}, to compute all $n$ $k$-similar
(or $k$-diverse) solutions at once. The distance function is
implemented in C++; in that sense, Online Method 3 allows for
finding similar/diverse solutions when the distance function cannot
be defined in ASP.
\end{itemize}

All the methods are sound, assuming that the ASP formulations are
correct. Online Method~1 is complete; however, Online Methods~2
and~3 are not because the computation of the similar/diverse
solutions depend on the first solution computed by {\sc Clasp}.

\item We have illustrated the applicability of these approaches on
two sorts of problems: phylogeny reconstruction (based on the ASP
encoding of the problem as in \cite{bro07}) and planning (based on
the ASP encoding of the Blocks World as in
\cite{Erdem-Theory-2002}).

\begin{itemize}
\item
For phylogeny reconstruction, we have defined novel distance
measures for a set of phylogenies (Section~\ref{ssec_DistMeasPhyl}),
described how the offline method and the online methods are applied
to find similar/diverse phylogenies
(Section~\ref{ssec_CompSimiDivePhy}), and compared the efficiency
and effectiveness of these methods on the family of Indo-European
languages studied in \cite{bro07} (Section~\ref{ssec_ExpeResu}).
Since there is no phylogenetic system that helps experts to analyze
phylogenies by comparing them, this particular application of our
methods also plays a significant role in phylogenetics. In fact,
Offline Method and Online Method 3 are integrated in the
phylogenetics system {\sc Phylo-ASP}~\cite{erdem09}.

\item
For planning, we have considered the action-based Hamming distance
of~\cite{SrivastavaNguyenEtAl-Domain-2007} to measure the distance
among plans, and compared the efficiency and effectiveness of the
offline method and the online  methods on some Blocks World problems
(Section~\ref{sec_CompSimiDivePlan}).

\end{itemize}

\end{itemize}

Finding similar/diverse solutions has earlier been studied in the
context of propositional logic \cite{Distance-SAT-99}, constraint
programming \cite{heb05,heb07}, and automated planning
\cite{SrivastavaNguyenEtAl-Domain-2007}. These studies consider the
Hamming distance \cite{ham50} as a measure to compute distances
between solutions. Unlike the problems studied in related work, the
problems we have studied are not confined to polynomial-time
distance functions with polynomial range. A more detailed discussion
on related work is presented in Section~\ref{sec_Related}.

\section{Answer Set Programming}\label{sec_ASP}

We study finding similar/diverse solutions in the context of Answer
Set Programming (ASP)~\cite{Lif-What-08}---a new declarative
programming paradigm where the idea is to represent a combinatorial
search problem as a ``program'' whose models (called ``answer sets''
\cite{gel91b}) correspond to the solutions. This is in the vein of
SAT solving, which became popular after a surprising success in the
area of planning \cite{kaut-selm-92}, but offers in comparison
features like variables ranging over domain elements, easy
definition of transitive closure, and nonmonotonic negation.
Furthermore, a range of special constructs, such as aggregates,
weight constraints and priorities, that are useful in practical
applications are supported by various ASP solvers; for more
discussion, see Section~\ref{sec_Related}.

Before we proceed discussing
our methods for finding similar/diverse solutions in ASP, let us present
the syntax of the kind of programs considered in this paper, and define
the concept of an answer set for such programs.\footnote{Answer sets are
defined for programs of a more general form that may contain classical
negation~$\neg$ and disjunction~\cite{gel91b} and nested
expressions~\cite{lif99d} in heads of rules as well. See
\cite{Lifschitz10} for  definitions of answer sets.}

\paragraph{Programs}
The syntax of formulas, rules and programs is defined as follows.
{\sl Formulas} are formed from propositional atoms and 0-place
connectives~$\top$ and~$\bot$  using negation (written as \no\/),
conjunction (written as a comma) and disjunction (written as a semicolon).

A {\em rule} is an expression of the form
\beq
F \lar G
\eeq{rule}
where $F$ is an atom or $\bot$, and $G$ is a formula; $F$ is called the {\em head}
and $G$ is called the {\em body} of the rule. A rule of the form $F\lar\top$ will be
identified with the formula~$F$. A rule of the form $\bot\lar F$ (called a {\em constraint})
will be abbreviated as $\lar F$.

A {\em (normal nested) program} is a finite set of rules. If bodies
of all rules in a program are of the form
$$A_1,\dots,A_m,\no\ A_{m+1},\dots,\no\ A_n$$
then the program is a {\em normal program}. A program is {\em positive}
if it does not contain any negation.

\paragraph{Answer Sets}
To define the concept of an answer set for a program,
let us first define the satisfaction relation and the reduct of a program.

The {\sl satisfaction relation} $X\models F$ between a set $X$ of atoms
and a formula~$F$ is defined recursively, as follows:
\begin{itemize}
\item for an atom~$A$, $X\models A$ if $A\in X$
\item $X\models \top$
\item $X\not\models \bot$
\item $X\models (F,G)$ if $X\models F$ and $X\models G$
\item $X\models (F;G)$ if $X\models F$ or $X\models G$
\item $X\models \no\ F$ if $X\not\models F$.
\end{itemize}
We say that $X$ {\sl satisfies} a program $\Pi$ (symbolically, $X \models \Pi$)
if, for every rule~$F\lar G$ in~$\Pi$, $X \models F$ whenever $X \models~G$.

The {\sl reduct} $F^X$ of a formula $F$ with respect to a set $X$ of atoms
is defined recursively, as follows:
\begin{itemize}
\item if $F$ is an atom or a 0-place connective then $F^X=F$
\item $(F,G)^X=F^X,G^X$
\item $(F;G)^X=F^X;G^X$
\item $(\no\ F)^X=
\cases{
\bot\ , & if $X \models F$, \cr
\top\ , & otherwise.}$
\end{itemize}
The {\sl reduct} $\Pi^X$ of a program $\Pi$ with respect to~$X$ is the set
of rules
$$
F^X \lar G^X
$$
for all rules~$F\lar G$ in $\Pi$.

Let us first define the answer set for a program $\Pi$ that does not contain
negation. We say that $X$ is an {\em answer set} for~$\Pi$, if $X$
is minimal with respect to set inclusion ($\subseteq$) among the sets of atoms that satisfy $\Pi$.  For instance,
the set~$\{p\}$ is the answer set for the program consisting of the single rule
\beq
p \lar .
\eeq{ex:reduct-p}

Now consider a program~$\Pi$ that may contain negation.
A set~$X$ of atoms is an {\em answer set} for~$\Pi$
if it is the answer set for the reduct~$\Pi^X$.  For instance, the reduct
of the program
\beq
\ba l
p \lar \no\ \no\ p
\ea
\eeq{ex:p-or-np}
relative to~$\{p\}$ is (\ref{ex:reduct-p}). Since $\{p\}$ is the
answer set for (\ref{ex:reduct-p}), $\{p\}$ is an answer set for program
(\ref{ex:p-or-np}). Similarly, $\{\}$ is an answer set for program
(\ref{ex:p-or-np}) as well.

\paragraph{Representing a Problem in ASP}
The idea of ASP is to represent a computational problem as a program
whose answer sets correspond to the solutions of the problem, and to
find the answer sets for that program using an answer set solver.

When we represent a problem in ASP, two kinds of rules
play an important role: those that ``generate'' many answer sets corresponding
to ``possible solutions'', and those that can be used to ``eliminate'' the
answer sets that do not correspond to solutions. Rules~(\ref{ex:p-or-np}) are
of the former kind: they generate the answer
sets $\{p\}$ and~$\{\}$. Constraints are of the latter kind. For instance,
adding the constraint
$$\lar p$$
to program (\ref{ex:p-or-np}) eliminates the answer sets for
the program that contain~$p$.

In ASP, we use special constructs of the form
\beq
 \{A_1,\dots,A_n\}^c
\eeq{choice}
(called {\em choice expressions}), and of the form
\beq
l \le \{A_1,\dots,A_m\} \le u
\eeq{cardinality}
(called {\em cardinality expressions}) where each $A_i$ is an atom and $l$ and $u$ are
nonnegative integers denoting the ``lower bound'' and the ``upper bound''~\cite{sim02}.
Programs using these constructs can be viewed as abbreviations for normal nested programs
defined above, due to~\cite{fer05}.  For instance, the following program
$$
\{p\}^c \lar
$$
stands for program~(\ref{ex:p-or-np}). The constraint
$$
\lar\ 2 \le \{p,q,r\}
$$
stands for the constraints
$$
\ba l
\lar\ p,q \\
\lar\ p,r \\
\lar\ q,r .
\ea
$$

Expression~(\ref{choice}) describes subsets of $\{A_1,\dots,A_n\}$.
Such expressions can be used in heads of rules to generate many answer sets.
For instance, the answer sets for the program
\beq
\{p,q,r\}^c\ \lar\
\eeq{ex:p-q-r}
are arbitrary subsets of~$\{p,q,r\}$.
Expression~(\ref{cardinality}) describes the subsets of the set
$\{A_1,\dots,A_m\}$ whose cardinalities are at least $l$ and at most~$u$.
Such expressions can be used in constraints to eliminate some answer sets.
For instance, adding the constraint
$$
\lar\ 2 \le \{p,q,r\}
$$
to program (\ref{ex:p-q-r}) eliminates the answer sets for
(\ref{ex:p-q-r}) whose cardinalities are at least 2.  Adding the
constraint
\beq
\lar\ \no\ (1 \le \{p,q,r\})
\eeq{ex:c1}
to program (\ref{ex:p-q-r}) eliminates the answer sets for
(\ref{ex:p-q-r}) whose cardinalities are not at least~1.

We abbreviate the rules
$$
\ba l
 \{A_1,\dots,A_m\}^c \lar \ii{Body} \\
\lar \no\ (l \le \{A_1,\dots,A_m\}) \\
\lar \no\ (\{A_1,\dots,A_m\} \le u)
\ea
$$
by
$$
l \le \{A_1,\dots,A_m\}^c \le u \lar \ii{Body} .
$$
For instance, rules~(\ref{ex:p-q-r}), (\ref{ex:c1}) and $\lar\ \no\
(\{p,q,r\} \le 1)$ can be written as
$$
1 \le \{p,q,r\}^c \le 1 \lar
$$
whose answer sets are the singleton subsets of~$\{p,q,r\}$.

\paragraph{Finding a Solution using an Answer Set Solver}
Once we represent a computational problem as a program whose answer sets correspond to
solutions of the problem,  we can use an answer
set solver to compute the solutions of the problem.
To present a program to an answer set solver, like \clasp,
we need to make some syntactic modifications.

The syntax of the input language of \clasp\ is more limited in some ways
than the class of programs defined above, but
it includes many useful special cases. For instance, the head of a rule
can be an expression of one of the forms
$$
\ba l
\{A_1,\dots,A_n\}^c\\
l\leq \{A_1,\dots,A_n\}^c\\
\{A_1,\dots,A_n\}^c\leq u\\
l\leq \{A_1,\dots,A_n\}^c\leq u
\ea
$$
but the superscript $^c$ and the sign $\leq$ are dropped.
The body can contain cardinality expressions but the sign $\leq$ is dropped.

In the input language of \clasp, {\tt :-} stands for
$\lar$, and each rule is followed by a period.

A group of rules that follow a pattern can be often described in a compact
way using ``(schematic) variables''. Variables must be capitalized.
For instance, the program $\Pi_n$
$$
p_i \lar \no\ p_{i+1}   \qquad (1 \leq i \leq n)
$$
can be presented to \clasp\ as follows:
\begin{verbatim}
index(1..n).
p(I) :- not p(I+1), index(I).
\end{verbatim}
Here {\tt index} is a ``domain predicate'' used to describe
the range of variable {\tt I}.

Variables can be also used ``locally'' to describe the list of formulas in
a cardinality expression.   For instance, the rule
$$1\leq \{p_1,\dots,p_n\}\leq 1$$
can be expressed in \clasp\ as follows
\begin{verbatim}
index(1..n).
1{p(I) : index(I)}1.
\end{verbatim}

\clasp\ finds an answer set for
a program in two stages: first it gets rid
of the schematic variables using a ``grounder'', like  \gringo, and
then it finds an answer set for
the ground program using a DPLL-like
branch and bound algorithm (outlined in Algorithm~\ref{alg_CLASP}).

\section{Computational Problems}\label{sec_CompProb}

We study various problems related to finding  similar/diverse solutions
to a computational problem $P$ formulated in ASP. For that, we assume that the
problem is represented as a normal (possibly nested) program $\cal P$
whose answer sets characterize solutions of the problem.
More precisely, let $\sol{P}$ denote the set of solutions of $P$ and let $\asp$ denote
the set of answer sets of $\cal P$. Then, there is a many-to-one mapping of $\asp$ onto
$\sol{P}$. Moreover, given an answer set of $\cal P$ the corresponding solution from
$\sol{P}$ can efficiently be extracted.
We also assume that a distance function that maps a set $S$ of solutions to a number
is defined, to measure how similar/diverse the solutions are in $S$.
To this end, we consider set-distance measures $\Delta : 2^{\sol{P}}\mapsto \mathbb{N}_0$ on
solutions for $P$.

We are mainly interested in two sorts of problems related to
computation of a diverse/similar collection of solutions:

\begin{itemize}
\item[]
{\sc $n$ $k$-similar solutions} (\resp{} {\sc $n$ $k$-diverse solutions})\\
Given an ASP program $\cal P$ that formulates a computational
problem $P$, a distance measure $\Delta$ that maps a set of
solutions for $P$ to a nonnegative integer, and two nonnegative
integers $n$ and $k$, decide whether a set $S$ of $n$ solutions for
$P$ exists such that $\Delta(S) \leq k$ (\resp $\Delta(S) \geq k$).

\item[]
{\sc $k$-close solution} (\resp{} {\sc $k$-distant solution})\\
Given an ASP program $\cal P$ that formulates a computational
problem $P$, a distance measure $\Delta$ that maps a set of
solutions for $P$ to a nonnegative integer, a set $S$ of solutions
for $P$, and a nonnegative integer $k$, decide whether some solution
$s$ ($s\not\in
 S$) for $P$ exists such that $\Delta(S \cup\{s\}) \leq k$ (\resp $\Delta(S \cup\{s\})
\geq k$).
\end{itemize}

For instance, suppose that the ASP program $\cal P$  describes the
phylogeny reconstruction problem for Indo-European languages as in
\cite{bro05}; so each answer set of $\cal P$ represents a phylogeny
for Indo-European languages. Using this ASP program with an existing
ASP solver, one can compute many phylogenies for the same input
dataset and with the same input parameters. Instead of analyzing all
of these phylogenies manually, a historical linguist may ask for,
for instance, three phylogenies whose diversity is at least 20 with
respect to some domain-independent or domain-dependent distance
function~$\Delta$; this problem is an instance of $n$~$k$-diverse
solutions problem where $n=3$ and $k=20$. On the other hand, a
historical linguist may have found two phylogenies $P_1$ and $P_2$
that are plausible, for instance, based on some archeological
evidence, and she may want to infer a similar phylogeny whose
distance from $\{P_1,P_2\}$ is at most 10; this problem is an
instance of $k$-close solution problem where $k=10$.

The first kind of problems above has two parameters, $n$ and $k$, so
we can fix one and try to minimize (\resp maximize)  the distance
between solutions to find the most similar  (\resp diverse)
solutions.

\begin{itemize}
\item[]
{\sc $n$ most similar solutions} (\resp{} {\sc $n$  most diverse solutions})\\
Given an ASP program $\cal P$ that formulates a computational
problem $P$, a distance measure $\Delta$ that maps a set of
solutions for $P$ to a nonnegative integer, and a nonnegative
integer $n$, find a set $S$ of $n$ solutions for $P$ with the
minimum (\resp maximum) distance $\Delta(S)$.

\item[]
{\sc maximal $n$ $k$-similar solutions} (\resp{} {\sc maximal $n$ $k$-diverse
  solutions})\\
Given an ASP program $\cal P$ that formulates a computational
problem $P$, a distance measure $\Delta$ that maps a set of
solutions for $P$ to a nonnegative integer, and a nonnegative
integer $k$, find a $\subseteq$-maximal set $S$ of at most $n$ solutions for $P$ such that
$\Delta(S) \leq k$ (\resp $\Delta(S) \geq k$) exists.

\end{itemize}

In the second class of problems, we can try to minimize (\resp
maximize) the distance $k$ between a solution and a set of
solutions, to find the closest (\resp most distant) solution.

\begin{itemize}
\item[]
{\sc closest solution}
(\resp~{\sc most distant solution})\\
Given an ASP program $\cal P$ that formulates a computational
problem $P$, a distance measure $\Delta$ that maps a set of
solutions for $P$ to a nonnegative integer, and a set $S$ of
solutions for $P$, find a solution $s$ ($s\not\in S$) for $P$ with
the minimum (\resp maximum) distance  $\Delta(S \cup\{s\})$.
\end{itemize}

We can generalize {\sc $k$-close solution} (\resp{} {\sc $k$-distant
solution}) problems to sets of solutions:

\begin{itemize}
\item[]
{\sc $k$-close set} (\resp{} {\sc $k$-distant set})\\
Given an ASP program $\cal P$ that formulates a computational
problem $P$, a distance measure $\Delta$ that maps a set of
solutions for $P$ to a nonnegative integer, a set $S$ of solutions
for $P$, and a nonnegative integer $k$, decide whether a set $S'$ of
solutions for~$P$ ($S'\neq S$) exists such that $|\Delta(S) -
\Delta(S')| \leq k$ (resp. $|\Delta(S) - \Delta(S')| \geq k$).
\end{itemize}

Usually an expert is interested in several kinds of problems to be
able to systematically analyze solutions. For instance, a historical
linguist may want to find three most diverse phylogenies; and after
identifying one particular plausible phylogeny among them, she may
want to compute another phylogeny that is the closest. An example of
such an analysis is shown in Section~\ref{ssec_ExpeResu} for
understanding the classification of Indo-European languages.

We note that the problems on
similar/diverse solutions from above can be analogously defined for
computation problems with multiple (or possibly none) solutions in
general, and in particular for such problems with \NP\ complexity.
Since ASP can express all \NP\ search problems \cite{mare-remm-03},
in fact similar/diverse solution computation for each such problem
can be formulated in the framework above (in fact with polynomial
overhead).

\section{Complexity Results} \label{sec_CompResu}

Before we discuss how the computational problems described in the
previous section can be solved in ASP, let us turn our attention to
the computational complexity of the problems presented in
Section~\ref{sec_CompProb}. In order to do so, we first make some
reasonable assumptions on some of the problem parameters.

In the following we assume that given an answer set $s$ of $\cal P$,
extracting a solution of $P$ from $s$ can be accomplished in time polynomial wrt.~the
size of $s$. Moreover, w.l.o.g.~we identify $s$ with the solution it encodes, and sets
$S\subseteq\sol{P}$ with corresponding sets of answer sets from $\asp$.

We assume that all numbers are given in binary and that the given number $n$
of different solutions to consider (respectively the size of the set $S$)
for instances of the problems {\sc $n$ $k$-similar solutions}, {\sc
maximal $n$ $k$-similar solutions}, and {\sc
$n$ most similar solutions}  is polynomial in the size of the input.
The same assumption applies to the size of the sets $S'$ to consider
in instances of {\sc $k$-close set} problems.

Furthermore, we consider distance measures $\Delta$
such that deciding whether $\Delta(S) \leq k$ (\resp whether $\Delta(S) \geq k$)
for a given $k$ is in \NP. Moreover, we assume that the value of
$\Delta(S)$ is bounded by an exponential in the size of $S$ (and
thus has polynomially many bits in the size of $S$). Thus, when
considering $\Delta$ as an input to a problem, we assume that it is
given as the description of a non-deterministic Turing machine
$M^\leq_\Delta$, or $M^\geq_\Delta$, or both, where $M^\leq_\Delta$
(\resp $M^\geq_\Delta$) nondeterministically
decides $\Delta(S) \leq k$ (\resp $\Delta(S) \geq k$)
in time polynomial in the length of its input $S$ and $k$.
Consequently, a witness for a computation of $M^\chi_\Delta$ on some
input $S$ and $k$, where $\chi\in\{\leq,\geq\}$ is a sequence of
configurations of $M^\chi_\Delta$, such that the input tape contains
$S$ and $k$ in the initial configuration, successive configurations
correspond to transitions of $M^\chi_\Delta$, and the final
configuration accepts.
In addition, we say that a $\Delta$ is \emph{normal}
if $|S|\leq 1$ implies $\Delta(S)=0$.

Under these assumptions, the computational complexity
(cf.~\cite{papa-94} for a background on the subject) of the problems
concerning the computation of similar/diverse solutions we are
interested in, is given in Table~\ref{tab:comp}. All entries are
completeness results (under usual reductions) and hardness holds
even if $\Delta(S)$ is computable in polynomial time. Moreover, the
results are the same for the `symmetric' problems, i.e., when {\sc
similar} is replaced with {\sc diverse}, and {\sc close} is replaced
with {\sc distant}, respectively. The proofs are included
in~\ref{sec:proofs}.

\begin{table}[t!]
\caption{Complexity results for computing similar solutions.}
\label{tab:comp}
\begin{center}
\begin{tabular}{c@{~}|@{~}l@{~}|@{~}l}
\#\ &\ Problem &\ Complexity\\
\hline
1 &\ {\sc $n$ $k$-similar solutions} &\ \NP \\
2 &\ {\sc $k$-close solution} &\ \NP \\
3 &\ {\sc maximal $n$ $k$-similar solutions}  &\ \FNPLog \\
4 &\ {\sc $n$ most similar solutions}  &\ \FPNP\ \ (\FNPLog)\\
5 &\ {\sc closest solution}  &\ \FPNP\ \ (\FNPLog) \\
6 &\ {\sc $k$-close set}  &\ \NP
\end{tabular}
\end{center}
\end{table}

\begin{theorem}\label{theo:n-k-sim}
Problem {\sc $n$ $k$-similar solutions} (resp.~{\sc $n$ $k$-diverse
solutions}) is \NP-complete. Hardness holds even if $\Delta(S)$ is
computable in constant time and for any normal $\Delta$.
\end{theorem}

Membership for problem {\sc $n$ $k$-similar solutions} (resp.~{\sc
$n$ $k$-diverse solutions}) follows from the fact that we can guess
not only a candidate set $S$ via the program $\cal P$ (since $S$ is polynomially bounded) but
also a witness for $\Delta(S) \leq k$ (\resp $\Delta(S) \geq k$),
and check in polynomial time whether every $s\in S$ is a solution
and that $\Delta(S) \leq k$ (\resp $\Delta(S) \geq k$). For
hardness, one simply reduces answer-set existence for normal,
propositional programs to this problem, which is an \NP-complete
problem.
For a hardness result resorting to partial Hamming distance, one can
confer~\cite{Distance-SAT-99}.

In our experiments with phylogeny reconstruction, by
Theorem~\ref{theo:n-k-sim}, we know that deciding the existence of
$n$ $k$-similar (\resp $k$-diverse) phylogenies is \NP-complete, if
the distance measure is the nearest neighbor interchange
distance~\cite{DasGuptaHeEtAl-distances-1997} whose computation is
beyond polynomial time, or if the distance measure is the nodal
distance or comparison of descendants distance (both defined in
Section~\ref{ssec_DistMeasPhyl})
that are computable in polynomial time. Also, in planning, if we
consider the Hamming distance~\cite{ham50} (as defined in
Section~\ref{sec_CompSimiDivePlan}),
which is polynomially computable, or the edit
distance involving transpositions, which is conjectured to be
NP-hard~\cite{baf98},
deciding the existence of $n$ $k$-similar (\resp
$k$-diverse) plans is \NP-complete. Therefore, it makes sense to
find similar/diverse phylogenies/plans using ASP.

By similar arguments we obtain \NP-completeness for problem {\sc
$k$-close solution} (resp.~{\sc $k$-distant solution}).

\begin{theorem}\label{theo:k-close}
Problem {\sc $k$-close solution} (resp.~{\sc $k$-distant solution})
is \NP-complete. Hardness holds even if $\Delta(S)$ is computable
in constant time and for any normal $\Delta$.
\end{theorem}

When looking for maximal sets of solutions, we face
a function problem; here we also assume a polynomial upper bound on the
size of the sets $S$ to consider (given by input $n$ and our
corresponding assumption).
Recall that function problems generalize decision problems asking for a
finite, possibly empty set of solutions of every problem instance. The
solutions to function problems can be computed by transducers, i.e.,
possibly nondeterministic Turing machines equipped with an output tape,
which contains a solution if the input is accepted. Note that if the
Turing machine is nondeterministic, then it computes a multi-valued
(partial) function. For instance, \FNP\ is the class of multi-valued
function problems that can be solved by a nondeterministic transducer in
polynomial time, such that a given solution candidate can be checked in
polynomial time.

In particular, {\sc maximal $n$ $k$-similar solutions} (resp.~{\sc $n$ maximal
$k$-diverse solutions}) is solvable in \FNPLog. Intuitively,
\FNPLog\ is the class of function problems solvable in polynomial
time using a nondeterministic Turing machine with output tape that
may consult once an oracle that computes the optimal value of an
optimization problem whose associated decision problem is solvable in
\NP, provided
that this value has logarithmically many bits in the size of the input (see,
e.g.,~\cite{chen-toda-95,eite-subr-99} for more information on
\FNPLog\ and other function classes used in this section).

\begin{theorem}\label{theo:max-k-sim}
Problem {\sc maximal $n$ $k$-similar solutions} (resp.~{\sc maximal
$n$ $k$-diverse solutions}) is \FNPLog-complete. Hardness holds even if
$\Delta(S)$ is computable in polynomial time.
\end{theorem}

Membership can be shown by computing the maximum cardinality of
a set of at most $n$ solutions $S$ using the oracle.
Obviously, computing the maximum cardinality $c$
of a set of at most $n$ solutions $S$ is an optimization
problem whose associated decision problem is the following:
decide whether a given $c$ (such that $c\leq n$) is the cardinality of a
set $S$ of (at most $n$) solutions. Since the latter problem is in \NP\
(guess $S$ and check in polynomial time whether $|S|=c$, $\Delta(S) \leq k$,
and every $s\in S$ is a solution), the optimization problem is amenable to the
oracle provided that the computed value (optimal $c$) has logarithmically
many bits in the size of the input.
Note that since $|S|$ is polynomially bounded in the size of the input,
it has logarithmically many bits as required. Once the optimal value is
computed, one can nondeterministically compute a set $S$ of respective
size together with a witness for $\Delta(S) \leq k$, and check in
polynomial time that this is indeed the case.

Hardness can be shown by a reduction of $X$-MinModel
(cf.~\cite{chen-toda-95}).
We remark that the slightly different problem asking for a polynomial-size
set $S$ of solutions such that $\Delta(S)$ is minimal (respectively maximal),
again under the assumption that this value has logarithmically many
bits, is also \FNPLog-complete. For this variant, hardness can be
shown, e.g., for $\Delta(S)$ that takes the minimal (respectively
maximal) Hamming distance between answer sets in $S$ on a subset of
the atoms; note that such a {\em partial Hamming distance} is a
natural measure for problem encodings, where the disagreement on
output atoms is measured.

\begin{theorem}\label{theo:n-most-sim}
Problem {\sc $n$ most similar solutions} (resp.~{\sc $n$  most
diverse solutions}) is \FPNP-complete, and \FNPLog-complete if the
value of $\Delta(S)$ is polynomial in the size of $S$. Hardness
holds even if $\Delta(S)$ is computable in polynomial time.
\end{theorem}

\FPNP-membership of {\sc $n$ most similar solutions} (resp.~{\sc $n$
most diverse solutions}) is obtained by first using the \NP-oracle
to compute the minimum distance using binary search (deciding
polynomially many {\sc $n$ $k$-similar solutions} problems). Then,
the oracle is used to compute some suitable $S$ in polynomial time. Hardness
follows from a reduction of the Traveling Salesman Problem (TSP).
Notably, if the distances are polynomial in the size of the input,
i.e., if the value of $\Delta(S)$ is polynomially bounded in the
size of $S$, then the problem is \FNPLog-complete.

Proceeding similarly as before, completeness for \FPNP
(resp.~\FNPLog{} if $\Delta(S)$ is small) is obtained for {\sc closest
solution} (and for {\sc most distant solution}):

\begin{theorem}\label{theo:most-close}
Problem {\sc closest solution}
(resp.~{\sc most distant
solution}) is \FPNP-complete, and \FNPLog-complete if the value of
$\Delta(S)$ is polynomial in the size of $S$. Hardness holds even if
$\Delta(S)$ is computable in polynomial time.
\end{theorem}

For the generalization of {\sc $k$-close solution} (resp.~of {\sc
$k$-distant solution}) to sets, namely {\sc $k$-close set}
(resp.~{\sc $k$-distant set}), \NP-completeness holds by similar
arguments as for the former problem(s):

\begin{theorem}\label{theo:k-close-set}
Problem {\sc $k$-close set} (resp.~{\sc $k$-distant set}) is
\NP-complete. Hardness holds even if $\Delta(S)$ is computable
in constant time and for any normal $\Delta$.
\end{theorem}

\paragraph{Discussion} The results above, summarized in Table~\ref{tab:comp}, show that
computing similar solutions is intractable in general. This already
holds under the reasonable assumption that the distance measure $\Delta$
is normal, where all considered decision problems are \NP-complete.

The precise complexity characterization of the search problems
({\sc maximal $n$ $k$-similar solutions}, {\sc $n$ most similar solutions},
and {\sc closest solution}) reveals some information about the type
of algorithm we can expect to be suitable for solving these
problems in practice (for background, see \cite{chen-toda-95} and references
therein). In particular, we may not expect that they can be solved
by parallelization to \NP-problems in polynomial time, i.e., solve in
parallel polynomially many \NP-problems, e.g., SAT instances, and then
combine the results. On the other hand, for problem {\sc maximal
$n$  $k$-similar solutions} this is possible under randomization, i.e.,
with high probability of a correct outcome, due to the characteristics
of \FNPLog, while this is not the case for the problems {\sc $n$ most
  similar solutions} and {\sc closest solution} in the general case.
Rather, the results suggest that consecutive, dependent calls to \NP\
oracles are needed. Intuitively, backtracking-style algorithms,
which explore the search space to find solutions and then see to
(dis)prove optimality by finding better solutions, appropriately reflect
adaptivity.

However, from a worst-case complexity perspective, a simple realization
of such a scheme may not be optimal, as far too many solution
improvements (exponentially resp.\ polynomially many under ``small''
distance values) may happen until an optimal solution is found; here a
two phase algorithm (first compute the optimal solution cost in binary
search and then a solution of that cost, e.g.,\ with backtracking)
gives better guarantees.
In practice, one may intertwine bound and solution computation and
conduct a binary search over computations of solutions within a given bound.

In the next section, we consider first solving the search problem
analog of the decision problem {\sc $n$ $k$ similar solutions},
using different approaches, ranging from declarative encodings in
ASP over the explicit respectively implicit set of solutions, to a
generalized backtracking algorithm for evaluation ASP programs. We
then consider solving the related search problems  {\sc $n$ most
similar solutions} and {\sc maximal $n$ $k$-similar solutions} based
on the above considerations.  Finally, we discuss how we can solve
the problems {\sc $k$-close solution}, {\sc closest solution} and
{\sc $k$-close set} utilizing the methods introduced for {\sc $n$
$k$ similar solutions} and its variants.

\section{Computing Similar/Diverse Solutions}\label{sec_CompSimiDiveSolu}

Now we have a better understanding of the computational problems,
let us present our computational methods to find
$n$~$k$-similar/diverse solutions, $n$~most similar/diverse
solutions and maximal $n$ $k$-similar/diverse solutions for a given
computational problem $P$. Since the computation of similar
solutions and diverse solutions are symmetric, for simplicity, let
us only focus on the problems related to similarity. In the
following, suppose that the problem $P$ is described by an ASP
program {\tt Solve.lp}.

\subsection{Computing $n$ $k$-Similar Solutions}
\label{ssec_compsimisolu}

To compute a set of $n$ solutions whose distance is at most $k$, we
introduce an offline method and three online methods.

\subsubsection{Offline Method}\label{ssec_OfflMeth}

In the offline method, we compute the set $S$ of all the solutions
for $P$ in advance using the ASP program {\tt  Solve.lp}, with an
existing ASP solver. Then, we use some clustering method to find
similar solutions in $S$. The idea is to form clusters of $n$
solutions, measure the distance of each cluster, and pick the
cluster whose distance is less than or equal to $k$.

We can compute clusters of $n$ solutions whose distance is at most
$k$ by means of a graph problem: build a complete graph $G$ whose
nodes correspond to the solutions in $S$ and edges are labeled by
distances between the corresponding solutions; and decide whether
there is a clique $C$ of size $n$ in $G$ whose weight (i.e., the
distance of the set of solutions denoted by the weight of the
clique) is less than or equal to $k$. The set of vertices in the
clique represents $n$ $k$-similar solutions.

The weight of a clique (or the distance $\Delta$ of the solutions in
the cluster) can be computed as follows:  Given a function $d$ to
measure the distance between two solutions, let $\Delta(S)$ be the
maximum distance between any two solutions in $S$. Then $n$
$k$-similar solutions can be computed by
Algorithm~\ref{alg_offlsolu}
where the graph $G$ is built as follows: nodes correspond to
solutions in $S$, and there is an edge between two nodes $s_1$ and
$s_2$ in $G$ if $d(s_1,s_2) \leq k$. Nodes of a clique of size $n$
in this graph correspond to $n$~$k$-similar solutions. Such a clique
can be computed using the ASP formulation in~\cite{Lif-What-08}, or
one of the existing exact/approximate algorithms discussed
in~\cite{Gutin-Handbook-2003}.

\begin{algorithm}[t!]
\caption{Offline Method}
\label{alg_offlsolu}
\begin{algorithmic}

\REQUIRE{A set $S$ of solutions, a distance function $d: S \times S
\mapsto \mathbb{N}$}, and two nonnegative integers $n$ and $k$.

\ENSURE{A set $C$ of $n$ solutions whose distance is at most~$k$.}


\STATE $V \leftarrow\ $ Define a set of $|S|$ vertices, each
denoting a unique solution in $S$;

\STATE $E = \{\{v_i,v_j\}\ |\ v_i\neq v_j,\ v_i, v_j\ \mathrm{denote}\
s_i, s_j \in S, \ d(s_i,s_j) \leq k\}$;

\STATE $C \leftarrow$ Find a clique of size $n$ in $\langle V,E
\rangle$;

\RETURN C

\end{algorithmic}
\end{algorithm}

\subsubsection{Online Method 1: Reformulation}\label{ssec_OnliMeth1}

Instead of computing all the solutions in advance as in the offline
method, we can compute $n$~$k$-similar solutions to the given
problem $P$ on the fly. First we reformulate the ASP program {\tt
Solve.lp} in such a way to compute $n$-distinct solutions; let us
call the reformulation as {\tt SolveN.lp}. Such a reformulation can
be obtained from {\tt Solve.lp} as follows:

\begin{enumerate}
\item
We specify the number of solutions: {\tt solution(1..n).}
\item
In each rule of the program { \tt Solve.lp}, we replace each atom
{\tt p(T1,T2,...,Tm)} (except the ones specifying the input) with
{\tt p(N,T1,T2...,Tm)}, and add to the body {\tt solution(N)}.
\item
Now we have a program that computes $n$ solutions. To ensure that
they are distinct, we add a constraint which expresses that every
two solutions among these $n$ solutions are different from each
other.
\end{enumerate}

\begin{figure}[t!]
\begin{center}
\includegraphics[scale=.45]{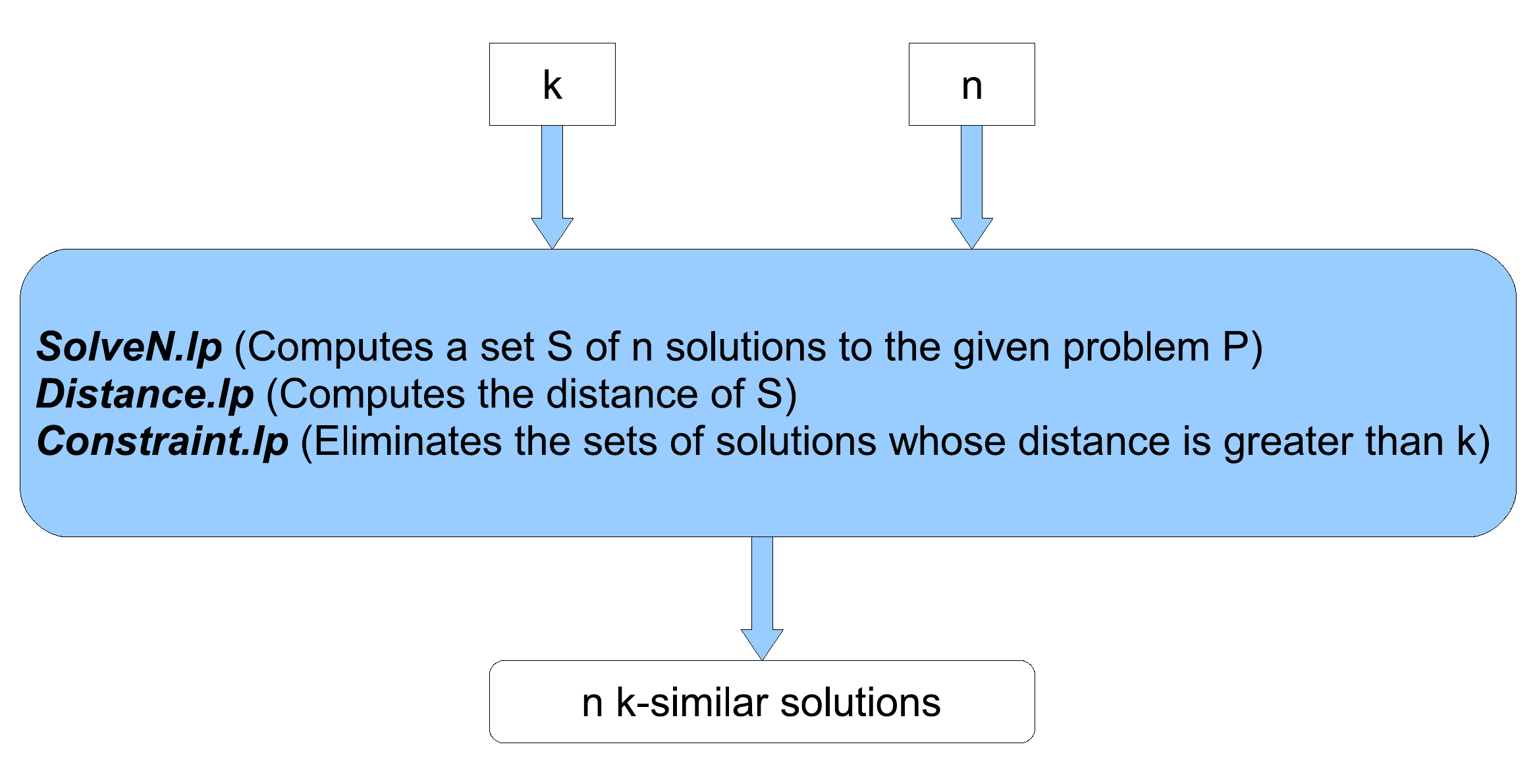}
\end{center}
\caption{Computing $n$ $k$-similar solutions, with Online Method 1.}
\label{fig_FlowChartOnline1-1}
\end{figure}

Next we describe the distance function $\Delta$ as an ASP program,
{\tt Distance.lp}. In addition, we represent the constraints on the
distance function (e.g., the distance of the solutions in $S$ is at
most $k$) as an ASP program {\tt  Constraint.lp}. Then we can
compute $n$-distinct solutions for the given problem $P$ that are
$k$-similar, by one call of an existing ASP solver with the program
{\tt  SolveN.lp} \ $\cup$ \ {\tt Distance.lp} \ $\cup$ \ {\tt
Constraint.lp}, as shown in Fig.~\ref{fig_FlowChartOnline1-1}. Let
us give an example to illustrate Online Method 1.
\begin{example}
Suppose that we want to compute $n$~$k$-similar cliques in a graph.
Assume that the similarity of two cliques is measured by the Hamming
Distance: the distance between two cliques $C$ and $C'$ is equal to
the number of different vertices, $(C\setminus C') \cup (C'\setminus
C)$. The distance of a set $S$ of cliques can be defined as the
maximum distance among any two cliques in $S$.

The clique problem can be represented in ASP ({\tt Solve.lp}) as
in~\cite{Lif-What-08}, also shown in~\ref{sec:pgms}
(Fig.~\ref{fig:clique}). The reformulation ({\tt SolveN.lp}) of this
ASP program as described above can be seen in Fig.~\ref{fig:cliqueN}
of~\ref{sec:pgms}. This reformulation computes $n$ distinct cliques.

The Hamming Distance between any two cliques can be represented by
the ASP program ({\tt Distance.lp}) shown in Fig.~\ref{fig:hamming}
of~\ref{sec:pgms}.

Finally, Fig.~\ref{fig:cons} shows the constraint ({\tt
Constraint.lp}) that eliminates the sets whose distance is above
$k$.

An answer set for the union of these three programs, {\tt SolveN.lp}
\ $\cup$ \ {\tt Distance.lp} \ $\cup$ \ {\tt Constraint.lp},
corresponds to $n$~$k$-similar cliques.

\end{example}

\subsubsection{Online Method 2: Iterative
Computation}\label{ssec_OnliMeth2}

This method does not modify the given ASP program {\tt  Solve.lp} as
in Online Method 1, but still formulates the distance function and
the distance constraints as ASP programs. The idea is to find
similar/diverse solutions iteratively, where the $i$'th solution is
$k$-close to the previously computed $i-1$ solutions (Fig.
\ref{fig_FlowChartOnline2-1}). Here $n$ iterations lead to $n$
solutions whose distance is at most $k$ (i.e., $n$~$k$-similar
solutions).

\begin{figure}[t!]
\begin{center}
\includegraphics[scale=.45]{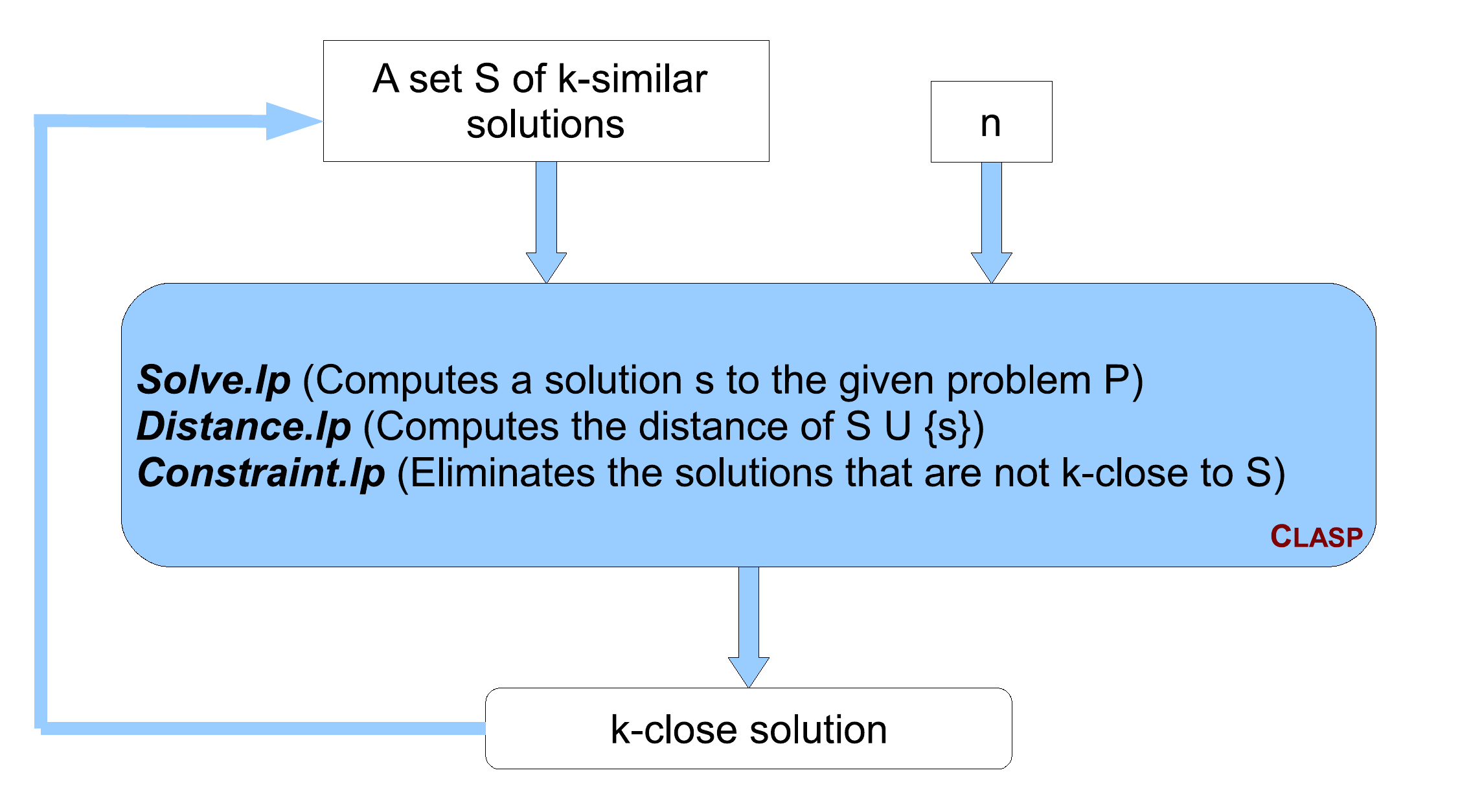}
\end{center}
\caption{Computing $n$ $k$-similar solutions, with Online Method 2.
Initially $S = \emptyset$. In each run, a solution is computed and
added to $S$, until $|S| = n$. The distance function and the
constraints in the program ensures that when we add the computed
solution to $S$, the set stays $k$-similar.}
\label{fig_FlowChartOnline2-1}
\end{figure}

Note that, like Offline Method and Online Method 1, this method is
sound; however, unlike Offline Method and Online Method 1, it is not
complete since computation of each solution depends on the
previously computed solutions. The method may not return
$n$~$k$-similar solutions (even it exists) if the previously
computed solutions comprise a bad solution set.

\subsubsection{Online Method 3: Incremental
Computation}\label{ssec_OnliMeth3}

This method is different from the other two online methods in that
it does not modify the ASP program {\tt Solve.lp} describing the
given computational problem $P$, it does not formulate the distance
function $\Delta$ and the distance constraints as ASP programs.
Instead,  modifies the search algorithm of an existing ASP solver in
such a way that the modified ASP solver can compute $n$~$k$-similar
solutions (Fig.~\ref{fig_nkSimilarIncremental}).
In this method, we modify the search algorithm of the ASP solver
\clasp\ (Version 1.1.3); the modified version is called \claspnk.
The given distance measure $\Delta$ is implemented as a {\tt C++}
program.

\begin{figure}[t!]
\begin{center}
\includegraphics[scale=.45]{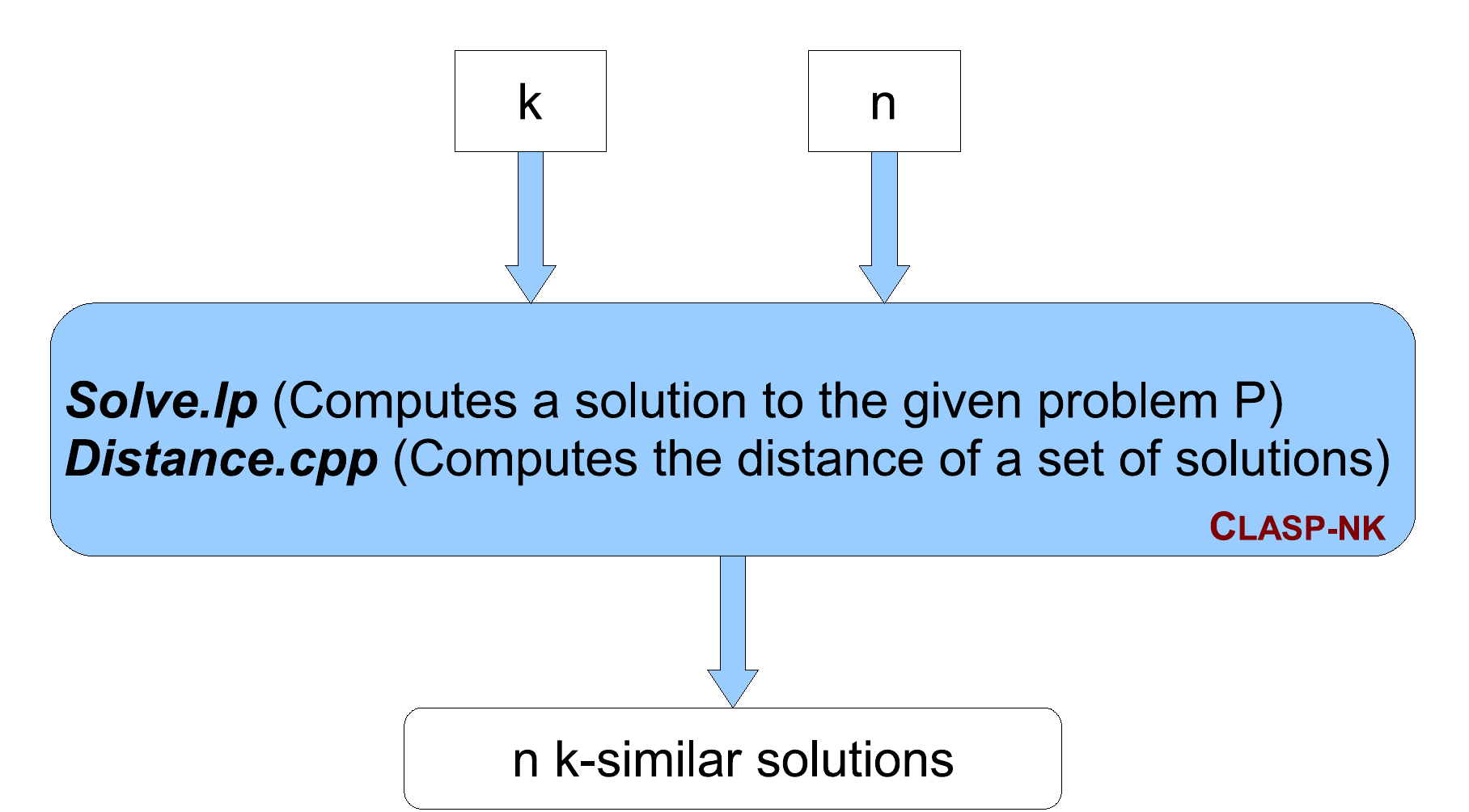}
\end{center}
\caption{Computing $n$~$k$-similar solutions, with Online Method 3.
\claspnk\ is a modification of the ASP solver \clasp, that takes
into account the distance function and constraints while computing
an answer set in such a way that \claspnk\ becomes biased to compute
similar solutions. Each computed solution is stored by \claspnk\
until a set of $n$~$k$-similar solutions is computed.}
\label{fig_nkSimilarIncremental}
\end{figure}

\begin{algorithm}[t!]
\caption{CLASP}
\label{alg_CLASP}
\begin{algorithmic}
\REQUIRE{An ASP program $\Pi$}

\ENSURE{An answer set $A$ for $\Pi$}

\STATE $A \leftarrow \emptyset$ \COMMENT{current assignment of
literals}

\STATE $\bigtriangledown \leftarrow \emptyset$ \COMMENT{set of
conflicts}

\WHILE{No Answer Set Found}

\STATE {\sc \textcolor{blue}{propagation($\Pi,A,\bigtriangledown$)}}
\COMMENT{propagate literals}

\IF{There is a conflict in the current assignment}

\STATE {\sc
\textcolor{blue}{resolve-conflict($\Pi,A,\bigtriangledown$)}}
\COMMENT{learn and update conflicts, and backtrack}

\ELSE{}

\IF{Current assignment does not yield an answer set}

\STATE {\sc \textcolor{blue}{select($\Pi,A,\bigtriangledown$)}}
\COMMENT{select a literal to continue search}

\ELSE{}

\RETURN A

\ENDIF \ENDIF \ENDWHILE
\end{algorithmic}
\end{algorithm}

Let us describe how we modified \clasp\ to obtain \claspnk. {\sc
clasp} performs a DPLL-like~\cite{DPLL-62,ClauseLearning-99} branch
and bound search to find an answer set for a given ASP program
(Algorithm \ref{alg_CLASP}): at each level, it ``propagates'' some
literals to be included in the answer set, ``selects'' new literals
to branch on, or ``backtracks'' to an earlier appropriate point in
search while ``learning conflicts'' to avoid redundant search.

\begin{algorithm}[t!]
\caption{CLASP-NK} \label{alg_CLASP2}
\begin{algorithmic}
\REQUIRE{An ASP program $\Pi$, \underline{nonnegative integers
$n$, and $k$}}

\ENSURE{\underline{A set $X$ of $n$ solutions that are $k$
similar ($n$ $k$-similar solutions)}}

\STATE $A \leftarrow \emptyset$ \COMMENT{current assignment of
literals}

\STATE $\bigtriangledown \leftarrow \emptyset$ \COMMENT{set of
conflicts}

\STATE \underline{$X \leftarrow \emptyset$} \COMMENT{computed
solutions}

\WHILE{\underline{$|X| < n$}}

\STATE $\underline{\ii{PartialSolution} \leftarrow A }$

\STATE \underline{$LowerBound \leftarrow$\ {\sc
distance-analyze}$(X,\ii{PartialSolution})$} \COMMENT{compute a
lower bound for the distance between any completion of a partial
solution and the set of previously computed solutions}

\STATE \textcolor{blue}{{\sc
{propagation}}($\Pi,A,\bigtriangledown$)} \COMMENT{propagate
literals}

\IF{Conflict in propagation \underline{{\sc or} $LowerBound >
k$}}

\STATE {\sc
\textcolor{blue}{resolve-conflict($\Pi,A,\bigtriangledown$)}}
\COMMENT{learn and update conflicts, and backtrack}

\ELSE{}

\IF{Current assignment does not yield an answer set} \STATE {\sc
\textcolor{blue}{select($\Pi,A,\bigtriangledown$)}}\COMMENT{select
a literal to continue search}

\ELSE{}

\STATE $X \leftarrow X \cup \{A\}$ \STATE $A \leftarrow \emptyset$
\ENDIF \ENDIF \ENDWHILE \RETURN X
\end{algorithmic}
\end{algorithm}

We modify \clasp's algorithm as shown in Algorithm~\ref{alg_CLASP2}
to obtain \claspnk: the underlined parts show these modifications.
To use \claspnk, one needs to prepare an options file,
\emph{NKoptions}, to describe the input parameters to compute $n$
$k$-similar phylogenies, such as the values $n$ and $k$, along with
the names of predicates that characterize solutions and that are
considered for computing the distance between solutions. Note that
since an answer set (thus a solution) is computed incrementally in
\claspnk, we cannot compute the distance between a partial solution
and a set of solutions with respect to the given distance function
$\Delta$. Instead, one needs to implement a heuristic function to
estimate a lower bound for the distance between any completion $s$
of a partial solution with a set $S$ of previously computed
solutions. If this heuristic function is admissible then it does not
underestimate the distance of $S\cup \{s\}$ (i.e., it returns a
lower bound that is less than or equal to the optimal lower bound
for the distance).

Note that similar to Online Method 2, this method is also sound but
not complete since each solution depends on all previously computed
solutions.

\subsection{Computing $n$ Most Similar Solutions}
\label{ssec_CompNMostSimiPhyl}

In the previous sections, we have described some computational
methods to solve the decision problem {\sc $n$ $k$-similar
solutions}. Let us discuss how we can solve the optimization problem
{\sc $n$ most similar solutions}. Let $NKSimilar(n,k)$ be a function
that returns---with one of the methods described in the previous
subsections---a set $S$ of $n$ solutions which is $k$-similar; or
returns empty set if no such set exists. Using this function, we can
find $n$~most~similar~solutions as follows: First we compute a lower
bound and an upper bound for the distance $k$ of a set of $n$
solutions. Then, we perform a binary search within these bounds to
find a set $S$ of solutions with the optimal value for $k$.
Computations of a lower bound and an upper bound are usually
specific to the particular problem. For instance, consider the
clique problem described in Section~\ref{ssec_OnliMeth1}. We can
find two most similar cliques in a graph, specifying the lower bound
as $0$ and the upper bound as the number of vertices in the graph
and using one of the methods described above.

\subsection{Computing Maximal $n$ $k$-Similar Solutions}
\label{ssec_CompMaxiKSimiPhyl}

Another optimization problem we are interested in is {\sc maximal
$n$ $k$-similar solutions}, which asks for a maximal set of
solutions whose distance is at most $k$. We can solve this problem
by modifying Online Method 2: start with a solution (computed using
{\tt Solve.lp}), then repeatedly find a solution which is $k$-close
to the previously found solutions until there does not exists such a
solution.  Recall that Online Method 2 iterates $n$ times; here the
iterations continue until no $k$-close solution is found. Since
Online Method 2 is incomplete, this method of computing maximal $n$
$k$-similar solutions is incomplete as well.

\subsection{Computing Close/Distant Solutions}
\label{ssec_compClosesolu}

We can solve the problem {\sc $k$-close solution} utilizing the
methods for {\sc $n$ $k$-similar solutions}. For instance, we can
modify Online Method 2: start with a set $S$ of solutions, then find
a solution which is $k$-close to $S$. Based on this modified method,
we can solve the problem {\sc closest solution}: we can compute a
lower bound and an upper bound for $k$, and find the optimal value
for $k$ by a binary search between these bounds as described in the
method for {\sc $n$ most similar solutions}.

Alternatively, for {\sc $k$-close solution}, we can modify the ASP
program ${\cal P}$ ({\tt Solve.lp}) that describes the computational
problem $P$, by adding constraints, to ensure that the answer sets
for ${\cal P}$ characterize solutions for $P$ except for the ones
included in the given set $S$ of solutions. Let us call the modified
ASP program ${\cal P}'$. Next, we define a distance measure
$\Delta'$ that maps a set of solutions for $P$ to a nonnegative
integer, in terms of the given measure $\Delta$ as follows:
$\Delta'(X) = \Delta(S\cup X)$. Then, we can use one of the
computational methods introduced for {\sc $n$ $k$-similar solutions}
with the ASP program ${\cal P}'$, the distance function $\Delta'$
and $n=1$. In a similar way,  we can find a solution to the problem
{\sc closest solution} utilizing the computational method for {\sc
$n$ most similar solutions}, with the ASP program ${\cal P}'$, the
distance measure $\Delta'$ and $n=1$.

We can solve the problem {\sc $k$-close set} using one of the
computational methods for {\sc $n$ $k$-similar solutions} as well.
For instance, we can use Online Method 1 with
$n=1,2,3,\ldots,m$, where $m$ is an upper bound on the number of
solutions for $P$, until a $k$-close set $S'$ of solutions is computed.
For each $n$, we reformulate the ASP program ${\cal P}$ to compute a set
$S'$ of $n$ solutions and add a constraint to this reformulation
to ensure that $S'\neq S$ when $n=|S|$; let us call this modified
reformulation ${\cal P}_n$. Then we try to find a $k$-close set $S'$
of $n$ solutions with the ASP program ${\cal P}_n$, and the distance measure
$\Delta''=|\Delta(S) - \Delta(S')|$.

Alternatively, we can use Online Method 2 or 3 with the ASP program
${\cal P}$, with the distance measure
$$
\Delta'''=\left\{
  \ba{ll}
    \infty                   & S=S' \\
    |\Delta(S) - \Delta(S')| & \textrm{otherwise}
  \ea\right.
$$
and with $n=1,2,3,\ldots,m$ where $m$ is an upper bound on the number of
solutions for $P$.


\section{Computing Similar/Diverse
Phylogenies}\label{sec_CompSimiDivePhyl}

Let us now illustrate the usefulness of our methods in a real-world
application: reconstruction of evolutionary trees (or phylogenies)
of a set of species based on their shared traits.  This problem is
important for research areas as disparate as genetics, historical
linguistics, zoology, anthropology, archeology, etc.. For example, a
phylogeny of parasites may help zoologists to understand the
evolution of human diseases~\cite{PhyEcologyBehave-91}; a phylogeny
of languages may help scientists to better understand human
migrations~\cite{PrehistoryOfAustralia-82}.

There are several software systems, such as {\sc
phylip}~\cite{phylip}, {\sc paup}~\cite{PAUP} or {\sc
Phylo-ASP}~\cite{erdem09}, that can reconstruct a phylogeny for a
set of taxonomic units, based on ``maximum
parsimony''~\cite{ReconstOfEvolutionaryTrees-64} or ``maximum
compatibility''~\cite{MethodDeduceBrachSequence-65} criterion. With
some of these systems, such as {\sc Phylo-ASP}, we can compute many
good phylogenies (most parsimonious phylogenies, perfect
phylogenies, phylogenies with most number of compatible traits,
etc.) according to the phylogeny reconstruction criterion. In such
cases, in order to decide the most ``plausible'' ones, domain
experts manually analyze these phylogenies, since there is no
available phylogenetic system that can analyze/compare these
phylogenies.

For instance, {\sc Phylo-ASP} computes 45 plausible phylogenies for
the Indo-European languages based on the dataset of \cite{bro07}. In
order to pick the most plausible phylogenies, in \cite{bro07}, the
historical linguist Don Ringe analyzes these phylogenies by trying
to cluster these phylogenies into diverse groups, each containing
similar phylogenies. In such a case, having a tool that reconstructs
similar/diverse solutions would be useful: with such a tool, an
expert can compute (instead of computing all solutions) few most
diverse solutions, pick the most plausible one, and then compute
phylogenies that are close to this phylogeny.

Let us show how our methods can be used for this purpose. Before
that, we define a phylogeny and some distance functions to
measure the similarity/diversity of phylogenies.

\subsection{Distance Measures for Phylogenies}
\label{ssec_DistMeasPhyl}

A {\em phylogeny} for a set of taxa is a finite rooted leaf-labeled
binary directed tree $\langle V,E\rangle$ with a set $L$ of leaves
($L\subseteq V$). The set $L$ represents the given taxonomic units,
whereas the set $V$ describes their ancestral units and the set $E$
describes the genetic relationships between them. The labelings of
leaves denote the values of shared traits at those nodes. We
consider distance measures that depend on topologies of phylogenies,
therefore, while defining them we discard these labelings.

There are various measures to compute the distance between two
phylogenies~\cite{NovelAlgoForComparingPhyTree-NyePietro-06,RobinsonFlouds-CompPhyTrees-81,HonEt-ImprovedPhyComp-00,BrachAndScore-94,DasGuptaHeEtAl-distances-1997}.
In the following, first we consider one of these domain-independent
functions, the nodal distance
measure~\cite{Bluis-NodalDistance-2003}, to compare two phylogenies;
and then we define a distance measure for a set of phylogenies based
on the nodal distances of pairwise phylogenies, to show the
applicability of our methods for finding $n$ $k$-similar
phylogenies. Then we define a novel distance function that measures
the distance of two phylogenies, and a distance function that
measures the distance of a set of phylogenies, taking into account
some expert knowledge specific to evolution. With this measure we
also show the effectiveness of our methods.

\subsubsection{Nodal Distance of Two Phylogenies}
\label{ssec_DiffNodaDist}

The {\em nodal distance $\ii{ND}_P(x,y)$ of two leaves} $x$ and $y$
in a phylogeny $P$  is defined as follows: First, transform the
phylogeny $P$ (which is a directed tree) to an undirected graph $G$
where there is an undirected edge $\{i,j\}$ in the graph for each
directed edge $(i,j)$ in the phylogeny. Then $\ii{ND}_P(x,y)$ is
equal to the length of the shortest path between $x$ and $y$ in the
undirected graph $G$. For example, consider the phylogeny, $P_1$ in
Fig.~\ref{fig_twotrees}; the nodal distance between $a$ and $b$ is
3, whereas the nodal distance between $b$ and $c$ is 2. Intuitively,
the nodal distance between two leaves in a phylogeny represents the
degree of their relationship in that phylogeny.

Given two phylogenies $P_1$ and $P_2$ both with same set $L$ of
leaves, the {\em nodal distance $D_n(P_1,P_2)$ of two phylogenies}
is calculated as follows:

$$
D_n(P_1,P_2) = \sum_{x,y \in L} |\ii{ND}_{P_1}(x,y) -
\ii{ND}_{P_2}(x,y)| .
$$

\noindent Here the difference of the nodal distances of two leaves
$x$ and $y$ represents the contribution of this pair of leaves to
the distance between the phylogenies.

\begin{proposition}
\label{pro_nodal} Given two phylogenies $P_1$ and $P_2$ with same
set $L$ of leaves and the same leaf-labeling function,
$D_n(P_1,P_2)$ can be computed in $O(|L|^2)$ time.
\end{proposition}

\begin{figure}[!t]
\begin{center}
\includegraphics[scale=.4]{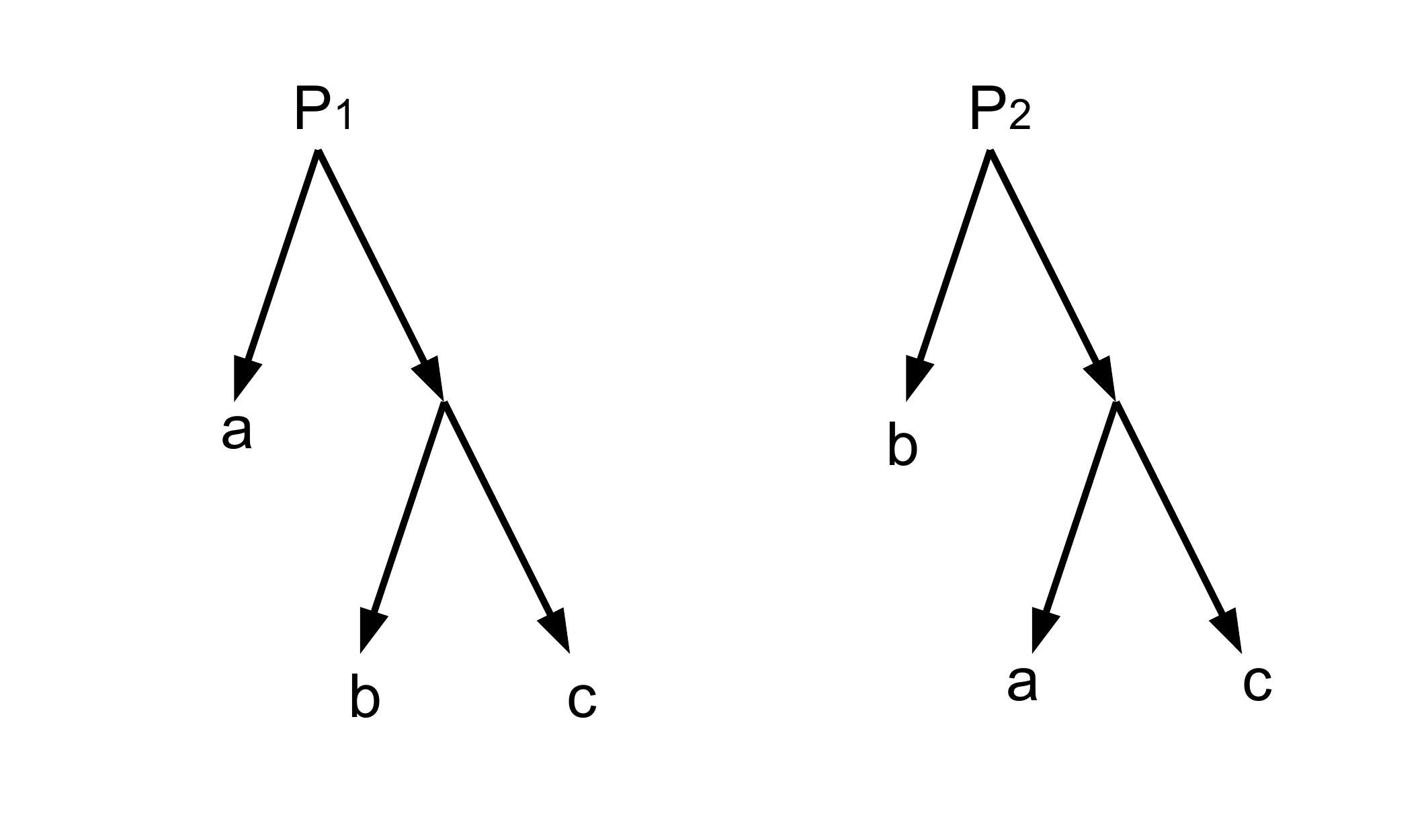} 
\end{center}
\caption{Two phylogenies $P_1 = (a,(b,c))$ and $P_2 = (b,(a,c))$}
\label{fig_twotrees}
\end{figure}

Example~\ref{table:nodaldexample} shows an example of computing the
nodal distance between two phylogenies. In that example, we
suppose that the phylogenies are
presented in the Newick format, where the sister sub-phylogenies are
enclosed by parentheses.  For instance, the first tree, $P_1$, of
Fig.~\ref{fig_twotrees} can be represented in the Newick format as
$(a, (b, c))$.

\begin{example}\label{table:nodaldexample}
In order to compute the nodal distance $D_n(P_1,P_2)$ between the
phylogenies $P_1=(a,(b,c))$ and $P_2=(b,(a,c))$ shown in
Fig.~\ref{fig_twotrees}, we compute the nodal distances of the pairs
of leaves, $\{a,b\}$, $\{a,c\}$ and $\{b,c\}$, and take the sum of
the differences:
\begin{center}
\begin{tabular}{cccc}
Pairs of leaves & Distance in $P_1$ & Distance in $P_2$ &
Difference\\
\hline
\{a,b\} & 3 & 3 & 0\\
\{a,c\} & 3 & 2 & 1\\
\{b,c\} & 2 & 3 & 1\\
\hline
\multicolumn{3}{l}{Total distance} & 2\\
\end{tabular}
\end{center}
In this case the distance between $P_1$ and $P_2$ is 2.
\end{example}

\subsubsection{Descendant Distance of Two Phylogenies}
\label{ssec_CompDescDistMeas}

Nodal distance measure computes the distance between two rooted
binary trees and does not consider the evolutionary relations
between nodes. In that sense, it is a domain-independent distance
measure for comparing phylogenies. A distance measure that takes
into account these relations and might give more accurate results.
Therefore, we define a new distance function based on our
discussions with the historical linguist Don Ringe. In particular,
we take into account the following domain-specific information in
phylogenetics: the similarities of phylogenies towards their roots
are more significant; and thus two phylogenies are more similar if
the diversifications closer to their roots are more similar.

For each vertex $v$ of a tree $T=\langle V,E \rangle$, let us define
the descendants of $x$ as follows:
$$
\ii{desc}_T(v) = \left\{
  \ba{ll}
     \{v\}                       & \textrm{$v$ is a leaf in $V$} \\
     \ii{desc}_T(u) \cup \ii{desc}_T(u') & \textrm{otherwise}\ (v,u),(v,u')
     \in E, u\neq u'
  \ea\right.
$$
and the depth of a vertex $v$ as follows:
$$
\ii{depth}_T(v) = \left\{
  \ba{ll}
     0                              & \textrm{$v$ is the root of $T$} \\
     1 + \ii{depth}_T(u)              & \textrm{otherwise} (u,v) \in E .
  \ea\right.
$$

To define the similarity of two phylogenies $T=\langle V,E\rangle$
and $T'=\langle V',E'\rangle$, let us first define the similarity of
two vertices $v\in V$ and $v'\in V'$:
$$
f(v,v') = \left\{
  \ba{ll}
  1 & \ii{desc}_T(v) \neq \ii{desc}_{T'}(v') \\
  0 & \textrm{otherwise}
  \ea\right.
$$

For every depth $i$ ($0 \le i \le \min\{\max_{v\in V}
\ii{depth}_T(v), \max_{v'\in V'} \ii{depth}_{T'}(v')\}$), let us
also define a weight function $\ii{weight}(i)$ that assigns a number
to each depth $i$. The idea is to assign bigger weights to smaller
depths so that two phylogenies are more similar if the
diversifications closer to the root are more similar. This is
motivated by the fact that reconstructing the evolution of languages
closer to the root is more important for historical linguists.

Now we can define the similarity of two trees $T = \langle
V,E\rangle$ and $T'=\langle V',E'\rangle$, with the roots $R$ and
$R'$ respectively, at depth $i$ ($0 \le i \le \min\{\max_{v\in V}
\ii{depth}(v), \max_{v'\in V'} \ii{depth}(v')\}$), by the following
measure:
$$
\ba{ll}
g(0, T, T') = & \ii{weight}(0) \times f(R,R') \\
g(i, T, T') = & g(i-1, T, T') + \\
& \ii{weight}(i) \times
         \sum_{x\in V, y\in V', \ii{depth}_T(x) = \ii{depth}_{T'}(y) = i}
           f(x,y),\quad i>0
\ea
$$
and the similarity of two trees as follows:
$$
D_l(T,T') = g(\min\{\max_{v\in V} \ii{depth}_T(v), \max_{v'\in V'}
\ii{depth}_{T'}(v')\}, T, T') .
$$

\begin{proposition}
\label{pro_compdesc} Given two trees $P_1$ and $P_2$ with same set
$L$ of leaves and the same leaf-labeling function, $D_l(P_1,P_2)$
can be computed in $O(|L|^3)$ time.
\end{proposition}

Example~\ref{table:compdescexample} shows an example of computing the
distance between two trees shown in Fig.~\ref{fig_twotrees}.

\begin{example}\label{table:compdescexample}
In order to compute the descendant distance $D_l(P_1,P_2)$
between the phylogenies $P_1=(a,(b,c))$ and $P_2=(b,(a,c))$ shown in
Fig.~\ref{fig_twotrees}, for each depth level, we multiply the
number of vertices that have different descendants with the weight
of that depth level. Then, we add up the products to find the total
distance between $P_1$ and $P_2$.
\begin{center}
\begin{tabular}{ccc}
Depth & Weight of Depth $i$ & Number of pairs of vertices that\\
&& have different descendant sets\\
\hline
0 & 2 & 0\\
1 & 1 & 4\\
2 & 0 & 3\\
\hline
Distance = & \multicolumn{2}{l}{$2\times 0 + 1\times4 + 0\times3 = 4$} \\
\end{tabular}
\end{center}
The descendant distance between $P_1$ and $P_2$ is 4.
\end{example}

\subsubsection{Distance of a Set of Phylogenies}

In the previous subsections, we defined distance functions for
measuring the distance between two phylogenies. However, the
problems that we defined in Section~\ref{sec_CompProb} requires a
distance function that measures the distance of a set of
phylogenies. We can define the distance of a set of phylogenies
based on the distances among pairwise phylogenies. For instance, the
distance of a set $S$ of phylogenies can be defined as the maximum
distance among any two phylogenies in $S$.

Let $D$ be one of the distance measures defined in the previous
subsection. Then, to be able to find similar phylogenies, the
distance of a set $S$ of phylogenies ($\Delta_D$) is defined as
follows:
$$
\Delta_D(S) = \max\{D(P_1,P_2)\ |\ P_1,P_2 \in S \} .
$$
To be able to find diverse phylogenies, the distance of a set $S$ of
phylogenies ($\Delta_D$) is defined as follows:
$$
\Delta_D(S) = \min\{D(P_1,P_2)\ |\ P_1,P_2 \in S \} .
$$

\subsection{Computing $n$ $k$-Similar/Diverse Phylogenies}
\label{ssec_CompSimiDivePhy}

Analogous to the $n$~$k$-similar (\resp diverse) solutions, we
define the $n$~$k$-similar (\resp diverse) phylogenies as follows:

\begin{itemize}
\item[]
{\sc $n$ $k$-similar phylogenies (\resp{} {\sc $n$ $k$-diverse phylogenies})}\\
Given an ASP program $\cal P$ that formulates a phylogeny
reconstruction problem $P$, a distance measure $\Delta_D$ that maps a
set of phylogenies for $P$ to a nonnegative integer, and two
nonnegative integers $n$ and $k$, decide whether a set $S$ of $n$
phylogenies exists such that $\Delta_D(S) \leq k$ (resp. $\Delta_D(S) \geq k$).
\end{itemize}

Recall that in order to compute $n$~$k$-similar (\resp diverse)
solutions we need an ASP program that computes a solution and a
distance measure. We consider the ASP program {\tt
phylogeny-improved.lp} described in \cite{bro07} as our main program
that computes a phylogeny; this program is shown Fig.s~\ref{fig:solve1} and~\ref{fig:solve2}
in~\ref{sec:pgms}. We represent the nodal distance $D_n$
(resp. the descendant $D_l$) of two phylogenies as the ASP program
in Fig.~\ref{fig:distance1} (resp. Figs.~\ref{fig:distance2-1} and
\ref{fig:distance2-2}) in~\ref{sec:pgms}. In addition, we
consider the program in Fig.~\ref{fig:constraint} that computes the
total distance of a set of solutions with $\Delta_D$ and eliminates
the ones whose total distance is greater than $k$.

For Offline Method, we compute all the phylogenies using { \tt
phylogeny-improved.lp}. Then we build a graph of phylogenies as in
Subsection~\ref{ssec_OfflMeth}. Then, we use the ASP program in
Fig.~\ref{fig:clique} in~\ref{sec:pgms} to find a clique of
size $n$ in the constructed graph. This clique corresponds to $n$
$k$-similar phylogenies.

For Online Method~1, we reformulate the main program {\tt
phylogeny-improved.lp} to obtain a program that computes
$n$~distinct phylogenies as in Section~\ref{ssec_OnliMeth1}. The
reformulation is shown in Fig.s~\ref{fig:solveN1}--\ref{fig:solveN3}
in~\ref{sec:pgms}.

For Online Method~3, we define a heuristic function to estimate a
low bound for the distance between any completion of a given partial
phylogeny and a complete phylogeny.

Let $P_c$ be any complete phylogeny, $P_p$ be any partial phylogeny
and $L_p$ be the set of pairs of leaves that appear in $P_p$.
Consider the nodal distance (Section~\ref{ssec_DiffNodaDist}) for
comparing two phylogenies. Then we can define a lower bound as
follows:
$$
{\cal LB}_n(P_p,P_c) = \sum_{x,y \in L_p} |\ii{ND}_{P_c}(x,y) -
\ii{ND}_{P_p}(x,y)| .
$$

\noindent This lower bound does not overestimate the distance
between a phylogeny and any completion of a partial phylogeny.
\begin{proposition}
\label{pro:lbn}
Given a partial phylogeny $P_p$ and a complete
phylogeny $P_c$, ${\cal LB}_n(P_p,P_c)$ is admissible, i.e., for every
completion $P$ of $P_p$,
$$
{\cal LB}_n(P_p,P_c) \leq D_n(P,P_c).
$$
\end{proposition}
Similarly, we can define an upper bound for the differences of nodal
distances measure as follows:
$$
{\cal UB}_n(P_p, P_c) = \sum_{x,y \in L_P}
|ND_{P_{c}}(x,y) - ND_{P_{p}}(x,y)| + ({l \choose 2} - {|L_p| \choose 2}) \times l.
$$
\noindent
where
$l$ denotes the number of leaves in the complete tree.

This upper bound does not underestimate the distance between a
phylogeny and any completion of a partial phylogeny.

\begin{proposition}
\label{pro:ubn}
Given a partial phylogeny $P_p$ and a complete
phylogeny $P_c$, ${\cal UB}_n(P_p,P_c)$ is admissible, i.e.,
for every completion $P$ of $P_p$, ${\cal UB}_n(P_p,P_c) \geq D_n(P,P_c).$
\end{proposition}
As regards the comparison of descendants distance measure, we could
not find a tight lower and upper bounds.  In our experiments, we
consider that the lower bound (\resp upper bound) between a complete
phylogeny and any completion of a partial phylogeny is 0
(\resp~$\infty$).

\subsection{Experimental Results for Phylogeny Reconstruction}
\label{ssec_ExpeResu}

We applied the offline method and the online methods described in
Section~\ref{ssec_compsimisolu} to reconstruct similar/diverse
phylogenies for Indo-European languages. We used the dataset
assembled by Don Ringe and Ann Taylor \cite{rin02}. As in
\cite{bro07}, to compute similar/diverse phylogenies, we considered
the language groups Balto-Slavic (BS), Italo-Celtic (IC),
Greco-Armenian (GA), Anatolian (AN), Tocharian (TO), Indo-Iranian
(IIR), Germanic (GE), and the language Albanian (AL). While
computing phylogenies, we also took into account some
domain-specific information about these languages.

In our experiments, we considered the distance measures described in
Section \ref{ssec_DistMeasPhyl} as in Section
\ref{ssec_CompSimiDivePhy}.

Below all CPU times are in seconds, for a workstation with a 1.5GHz
Xeon processor and 4x512MB RAM, running Red Hat Enterprise Linux
(Version 4.3).

\paragraph{Experiments with the Nodal Distance}
Let us first examine the results of experiments, considering the
distance measure $\Delta_n$, based on the nodal distance (Table
\ref{table:nodal}).
We present the results for the following computations: 2 most
similar solutions, 2 most diverse solutions, 3 most similar
solutions, 3 most diverse solutions, 6 most similar solutions. We
solve these optimization problems by iteratively solving the
corresponding decision problems ($n$ $k$-{\sc similar/diverse
solution}). For each method, we present the computation time, the
size of the memory used in computation, and the optimal value of
$k$.

Let us first compare the online methods. In terms of both
computation time and memory size, Online Method 3 performs the best,
and Online Method 2 performs better than Online Method 1. These
results conforms with our expectations. Online Method 1 takes as
input an ASP representation of computing $n$ $k$-similar/diverse
phylogenies, which is almost $n$ times as large  as the ASP program
describing the phylogeny reconstruction problem used in other
methods. Therefore, its computational performance may not be as good
as the other online methods. Online Method 2 relaxes this
requirement a little bit so that the answer set solver can compute
the solutions more efficiently: it takes as input an ASP
representation of phylogeny reconstruction, and an ASP
representation of the distance measure, and then computes
similar/diverse solutions one at a time. However, since the answer
set solver needs to compute the distances between every two
solutions, the computation time and the size of memory do not
decrease much, compared to those for Online Method 1. Online Method
3 deals with the time consuming computation of distances between
solutions, not at the representation level but at the search level.
In that sense, its computational performance is better than Online
Method 2.

The offline method takes into account the previously computed 8
phylogenies for Indo-European languages (with at most 17
incompatible characters), and computes similar/diverse solutions
using ASP as explained in Section \ref{sec_CompProb}. The offline
method is more efficient, in terms of both computation time and
memory, than Online Methods 1 and 2 since it does not compute
phylogenies. On the other hand, the offline method is less
efficient, in terms of both computation time and memory, than Online
Method 3, since it requires both representation and computation of
distances between solutions.

Here both the offline method and Online Method 1 guarantee finding
optimal solutions by iteratively solving the corresponding decision
problems. On the other hand, Online Methods 2 and 3 compute
similar/diverse solutions with respect to the first computed
solution, and thus may not find the optimal value for $k$, as
observed in the computation of 3 most similar phylogenies.

\begin{table}[t!]
\caption{Computing similar/diverse phylogenies using the nodal distance $\Delta_n$.} \label{table:nodal}
\begin{center}
\begin{tabular}{c||c||c|c|c}
Problem& Offline Method & \multicolumn{3}{c}{Online Methods}\\
 & & Reformulation & Iterative Comp. & Incremental Comp.\\
 & & ({\sc clasp}) & ({\sc clasp}, perl) & ({\sc clasp-nk})\\
\hline
2 most similar ~ & 12.39 sec. &  26.23 sec. & 19.00 sec.  & 1.46 sec.\\
($k=12$)               & 32MB       &  430MB      & 410MB       & 12MB    \\
               & $k=12$      &  $k=12$     & $k=12$     & $k=12$  \\
\hline
2 most diverse ~ & 11.81 sec. &  21.75 sec.   &  18.41 sec.    & 1.01 sec.       \\
($k=32$)              & 32MB       &  430MB        &  410MB         & 15MB     \\
              & $k=32$     &   $k=32$      &  $k=24$        & $k=32$        \\
\hline
3 most similar ~& 11.59 sec. &   60.20 sec.&   43.56 sec.   & 1.56 sec.    \\
($k=15$)              & 32MB       &   730MB     &   626MB         & 15MB  \\
              & $k=15$     &   $k=15$    &   $k=15$        & $k=16$      \\
\hline
3 most diverse ~& 11.91sec.  &   46.32 sec.&    44.67 sec.   & 0.96 sec.     \\
($k=26$)              & 32MB       &   730MB     &    626MB       &  15MB   \\
              & $k=26$     &   $k=26$    &    $k=21$      &  $k=26$      \\
\hline
6 most similar ~& 11.66sec. &   327.28 sec.  & 178.96 sec.    & 1.96 sec.   \\
($k=25$)             & 32MB       &     1.8GB   &  1.2GB          & 15MB             \\
              &  $k=25$  &  $k=25$ & $k=29$       &  $k=25$          \\
\end{tabular}
\end{center}
\end{table}

\paragraph{Experiments with the Nodal Distance}
Now, let us consider the distance measures $\Delta_l$, based on
preference over diversifications (Table \ref{table:CompDesc}): two
phylogenies are more similar if the diversifications closer to the
root are more similar. Here we consider the similarities of
diversifications until depth 3 (inclusive). We present the results
for the following computations: 2 most similar solutions, 3 most
diverse solutions, 6 most similar solutions.  In Table
\ref{table:CompDesc}, for each method, we present the computation
time, the size of the memory used in computation, and the optimal
value of $k$.
Unlike what we have observed in Table \ref{table:nodal}, the offline
method takes more time/space to compute similar/diverse solutions;
this is due to the computation of distances with respect to
$\Delta_l$ which requires summations.
Other results are similar to the ones presented in Table
\ref{table:nodal}.

\paragraph{Accuracy}
Let us compare the phylogenies computed by different distance
measures in terms of accuracy. In \cite{bro07}, after computing all
34 plausible phylogenies, the authors examine them manually and come
up with three forms of tree structures, and then ``filter'' the
phylogenies with respect to these tree structures. For instance, in
Group~1, the trees are of the form

\begin{center}
(AN, (TO, (AL, (IC, (a tree formed for GE, GA, BS, IIR)))));
\end{center}

\noindent in Group 2, the trees are of the form

\begin{center}
(AN, (TO, (IC, (a tree formed for GE, GA, BS, IIR, AL))));
\end{center}

\noindent in Group 3, the trees are of the form

\begin{center}
(AN, (TO, ((AL, IC), (a tree formed for GE, GA, BS, IIR)))).
\end{center}

\noindent The results of our experiments with the distance measure
$\Delta_l$ comply with these groupings. For instance, the 2 most
similar phylogenies computed by Online Method 1 are in Group~1;

\begin{center}
(AN, (TO, (IC, ((GE, AL), (GA, (IIR, BS)))))),\\
(AN, (TO, (IC, ((GE, AL), (BS, (IIR, GA)))))),
\end{center}

\noindent Phylogenies 7 and 8 of \cite{bro07}; both are in Group~1.
The 3 most diverse phylogenies computed by Online Method 2 are

\begin{center}
(AN, (TO, (IC, (AL, (GE, (GA, (IIR, BS))))))), \\
(AN, (TO, (AL, (IC, (GE, (GA, (IIR, BS))))))), \\
(AN, (TO, ((GE, (GA, (IIR, BS))), (AL, IC)))),
\end{center}

\noindent Phylogenies 10, 1, 40 of \cite{bro07}; all in different
groups. Likewise, the 6 similar phylogenies computed by our methods
fall into Group~2.

\bigskip
These results (in terms of computational efficiency and accuracy)
show the effectiveness of our methods in phylogeny reconstruction:
we can automatically compare many phylogenies in detail.

\begin{table}[t!]
\caption{Computing similar/diverse phylogenies using the descendant distance $\Delta_l$.}
\label{table:CompDesc}
\begin{center}
\begin{tabular}{c||c||c|c|c}
Problem& Offline Method & \multicolumn{3}{c}{Online Methods}\\
 & & Reformulation & Iterative Comp. & Incremental Comp.\\
 & & ({\sc clasp}) & ({\sc clasp}, perl) & ({\sc clasp-nk})\\
\hline
2 most similar ~ & 365.16 sec. &  16.11 sec.& 16.23 sec.             & 0.635 sec.\\
    $(k=18)$           & 4.2GB             &  236MB            & 212MB     & 22MB  \\
               & $k=18$             &$k=18$             &$k=18$     & $k=18$ \\
\hline
3 most diverse ~ &  368.59 sec.& 46.11 sec.& 44.21 sec.            & 1.014 sec.  \\
$(k=20)$               & 4.2GB             &  659MB            & 430MB    & 22MB \\
               & $k=20$             &$k=20$             &$k=20$    & $k=20$ \\
\hline
6 most similar ~ & 368.45 sec.        & 137.31 sec. & 212.59 sec. & 0.685 sec.    \\
$(k=18)$            & 4.2GB      & 1.8GB       &  1.1GB                          & 22MB   \\
                & $k=18$ & $k=18$ & $k=18$                                              & $k=20$ \\
\end{tabular}
\end{center}
\end{table}


\section{Computing Similar/Diverse Plans}
\label{sec_CompSimiDivePlan}

In order to show the applicability and effectiveness of our methods
to other domains, we extend our experiments further with the Blocks
World planning problems.

In these experiments, we study the following instance
of {\sc $n$ $k$-similar solutions} (\resp{} {\sc $n$ $k$-diverse solutions}):

\begin{itemize}
\item[]
{\sc $n$ $k$-similar plans (\resp{} {\sc $n$ $k$-diverse plans})}\\
Given an ASP program $\cal P$ that formulates a planning problem $P$,
a distance measure $\Delta_h$ that maps a set of
plans for $P$ to a nonnegative integer, and two nonnegative
integers $n$ and $k$, decide whether a set $S$ of $n$ plans for
$P$ exists such that $\Delta_h(S) \leq k$ (\resp $\Delta_h(S) \geq k$).
\end{itemize}

We take $\cal P$ as the ASP
formulation of the non-concurrent Blocks World as in~\cite{Erdem-Theory-2002}
to compute a plan of length at most $l$ (Fig.~\ref{fig:bw} in \ref{sec:pgms}),
together with an ASP description of the planning problem instance shown in
Fig.~\ref{fig_blocks}.

We define the distance $\Delta_h(S)$ of a set $S$ of plans as follows:
$$
\Delta_h(S) = \max\{D_h(P_1,P_2)\ |\ P_1,P_2 \in S, |P_1| \leq |P_2|\}
$$
based on the action-based Hamming distance $D_h$ defined
in~\cite{SrivastavaNguyenEtAl-Domain-2007} to measure the distance
between two plans. Intuitively, $D_h(P_1,P_2)$ is the number of
differentiating actions in each time step of two plans $P_1$ and
$P_2$. More precisely: let us denote a plan $X$ of length $l$ by a
function $\ii{act}_X$ that maps every nonnegative integer $i$
($1\leq i \leq l$) to the $i$'th action of the plan $X$, and let us
denote by $|X|$ the length of a plan $X$; then the
Hamming Distance $D_h(P_1,P_2)$ between two plans $P_1$ and $P_2$ such that
$|P_1| \leq |P_2|$ can be defined as follows:
$$
D_h(P_1,P_2) =
|\{i\ |\ \ii{act}_{P_1}(i) \neq \ii{act}_{P_2}(i), 1 \leq i \leq |P_1| \}|
+ |P_2| - |P_1|
$$
ASP formulations of the distance functions $D_h$ and $\Delta_h(S)$ are
presented in Fig.s~\ref{fig:distance-h2} and~\ref{fig:distance-plans}
in~\ref{sec:pgms}.

Consider, for instance, a planning problem in the Blocks World that asks for a plan of
length less than or equal to 7. Consider two plans, $P_1$ and $P_2$, that are characterized
by the functions $\ii{act}_{P_1}$ and $\ii{act}_{P_2}$ respectively, as follows:
$$
\ba{ll}
\ii{act}_{P_{1}}(1) = \ii{moveop}(b_1,\ii{Table}) & \ii{act}_{P_{1}}(2) = \ii{moveop}(b_2,b_1)\\
\ii{act}_{P_{1}}(3) = \ii{moveop}(b_4,\ii{Table}) & \ii{act}_{P_{1}}(4) = \ii{moveop}(b_3,b_2)\\
\ii{act}_{P_{1}}(5) = \ii{moveop}(b_4,b_3) & \ii{act}_{P_{1}}(6) = \ii{moveop}(b_5,b_4)\\
\\
\ii{act}_{P_{2}}(1) = \ii{moveop}(b_1,\ii{Table}) & \ii{act}_{P_{2}}(2) = \ii{moveop}(b_2,b_1)\\
\ii{act}_{P_{2}}(3) = \ii{moveop}(b_4,b_5) & \ii{act}_{P_{2}}(4) = \ii{moveop}(b_3,b_2)\\
\ii{act}_{P_{2}}(5) = \ii{moveop}(b_4,\ii{Table}) & \ii{act}_{P_{2}}(6) = \ii{moveop}(b_4,b_3)\\
\ii{act}_{P_{2}}(7) = \ii{moveop}(b_5,b_4)
\ea
$$

\noindent
The distance $D_h(P_1,P_2)$ between $P_1$ and $P_2$ is 4 since the actions at time steps
3, 5 and 6 are different and $P_2$ has an additional action (at time step 7).

To be able to apply our Online Method 3 with \claspnk\ to compute
$n$ $k$-similar plans of length at most $l$, we define a heuristic
function ${\cal LB}_h$ to estimate a lower bound for the distance
between a plan $P_c$ and any plan-completion of a ``partial'' plan
$P_p$. Intuitively, a partial plan consists of parts of a plan. Let
us characterize a partial plan $P_p$ by a partial function
$\ii{act}_{P_p}$ from $\{1,...,l\}$ to the set of actions; that is,
$\ii{act}_{P_p}$ is a function from a subset of $\{1,...,l\}$
to the set of actions. A {\em plan-completion of a partial plan} $P_p$ is a plan~$Y$
of length $l'$ ($l'\leq l$) for the planning problem $P$ such that
$\ii{act}_{Y}$ is an extension of $\ii{act}_{P_p}$ to $\{1,...,l'\}$.
Then we can define ${\cal LB}_h(P_p,P_c)$
for a partial plan $P_p$ and a plan $P_c$ as follows:
$$
\ba{ll}
{\cal LB}_h(P_p,P_c) =&
|\{i\ |\ \ii{act}_{P_p}(i) \neq \ii{act}_{P_c}(i), i\in \dom \ii{act}_{P_p}, 1\leq i \leq |P_c|\}|\\
& + \ |\{i\ |\ i\in \dom \ii{act}_{P_p}, |P_c| < i \leq l\}|\\
\ea
$$

In the Blocks World example above, consider a partial plan $P_p$
characterized by the function $\ii{act}_{P_p}$ as follows:
$$
\ba{ll}
\ii{act}_{P_{p}}(2) = \ii{moveop}(b_2,b_1) & \ii{act}_{P_{p}}(4) = \ii{moveop}(b_3,b_2)\\
\ii{act}_{P_{p}}(5) = \ii{moveop}(b_4,\ii{Table}) & \ii{act}_{P_{p}}(7) = \ii{moveop}(b_5,b_4)
\ea
$$

\noindent
The lower bound ${\cal LB}_h(P_p, P_1)$  for the distance between any completion of $P_p$ and $P_1$
is computed as follows:
$$
\ba{lcl}
{\cal LB}_h(P_p, P_1) & = & |\{i\ |\ \ii{act}_{P_{p}}(i) \neq \ii{act}_{P_{1}}(i), i \in \dom \ \ii{act}_{P_{p}} , 1 \leq i \leq 6\}|\\
& & + \ |\{i\ |\ i \in \dom \ \ii{act}_{P_{p}}, 6 < i \leq 7\}|\\
& = & |\{5\}| + |\{7\}| = 2 .
\ea
$$
One completion of $P_p$ is $P_2$. Note that  ${\cal LB}_h(P_p, P_1) \leq  D_h(P_1, P_2)$. Indeed,
the following proposition expresses that ${\cal LB}_h$ does not
overestimate the distance between $P_c$ and any plan-completion $X$ of $P_p$.

\begin{proposition}
\label{planning_lb}
For a partial plan $P_p$ and a plan $P_c$ for the planning problem $P$,
${\cal LB}_h(P_p,P_c)$ is admissible.
\end{proposition}

Similarly, to be able to apply our Online Method 3 with \claspnk\ to
compute $n$ $k$-diverse plans of length at most $l$, we define a heuristic
function ${\cal UB}_h(P_p,P_c)$ to estimate an upper bound for the distance between
a plan $P_c$ and any plan-completion of $P_p$:
$$
{\cal UB}_h(P_p,P_c) =
l - |\{i\ |\ \ii{act}_{P_p}(i) = \ii{act}_{P_c}(i), i\in \dom \ii{act}_{P_p}, 1\leq i \leq |P_c|\}| .
$$
For instance, for the partial plan $P_p$ and $P_1$ above,
$$
\ba{ll}
{\cal UB}_h(P_p, P_1) & = 7 - |\{i\ |\ \ii{act}_{P_{p}}(i) = \ii{act}_{P_{c}}(i), i \in \dom  \ii{act}_{P_{p}} , 1 \leq i \leq 6\}|\\
& =  7 - |\{2,4\}| = 7-2 = 5
\ea
$$
and ${\cal UB}_h(P_p, P_1) \geq D_h(P_1,P_2)$. Indeed,
the following proposition expresses that this upper bound function does not
underestimate the distance between $P_c$ and any plan-completion $X$ of $P_p$.
\begin{proposition}
\label{planning_ub}
For a partial plan $P_p$ and a plan $P_c$ for the planning problem $P$,
${\cal UB}_h(P_p,P_c)$ is admissible.
\end{proposition}

\begin{figure}[!t]
\begin{center}
\includegraphics[scale=.4]{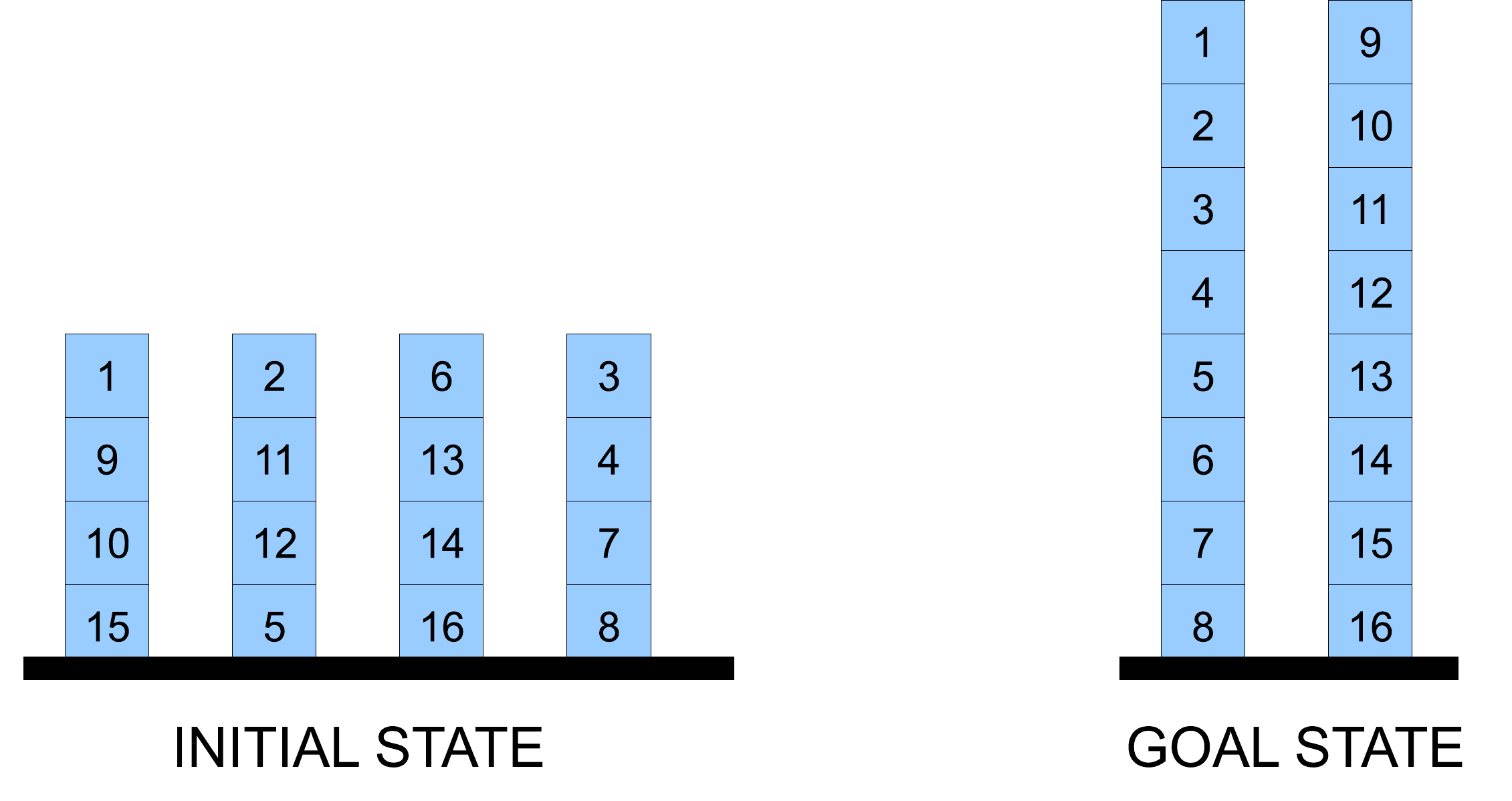}
\end{center}
\caption{Blocks World problem.} \label{fig_blocks}
\end{figure}

We performed some experiments with the ASP formulation, planning
problem, and distance measures above, to find 2~most~similar~plans,
2~most~diverse~plans, 3~most~similar~plans, 3~most~diverse~plans,
6~most~similar~plans.
Table~\ref{table:blocks} summarizes the results of these experiments.

It can be observed that
the planning problem in Fig.~\ref{fig_blocks} has too many
solutions to the problem (more than 50.000), and it is intractable
to compute all of them in advance and then the distances between
all pairwise solutions.
Therefore, instead of computing all the solutions in advance, we compute a
subset of them (around 200) which is small enough to construct a
distance graph, and apply our Offline Method in this way.
However, these 200 solutions are not diverse enough, and thus, although
we can find many very similar solutions, it is hard to find diverse solutions;
for instance, we can find 6 1-similar solutions but
we can find only 3 6-diverse solutions.

Online Method 1 performs the worst in comparison with the other
online methods, as in our experiments with phylogeny reconstruction
problems, due to the large ASP program (Fig.~\ref{fig:bw-n} in~\ref{sec:pgms})
used for computing $n$ distinct plans.

Online Method 2 is comparable with Online Method 3 in terms of
computing similar solutions. After computing a solution, computing a
1-similar solutions has a very small search space and {\sc Clasp}
can find a similar solution in a short time. On the other hand,
computing a 21-diverse solution has a huge search space. Therefore,
performance of computing diverse solutions with Online Method 2 is
worse than that of Online Method 3.

Online Method 3 deals with the Hamming distance computation at the
search level. In addition, it does not restart the search process to
compute a new solution; instead, it learns the conflicts caused by
distance difference while computing a new solution and backtracks to
approximate levels to compute similar/diverse solutions. Especially,
for the computation of diverse solutions, such a search strategy
creates a significant performance gain.

\begin{table}[t!]
\caption{Computing similar/diverse plans for the blocks world
problem. OM denotes ``Out of memory.''} \label{table:blocks}
\begin{center}
\begin{tabular}{c||c||c|c|c}
Problem& Offline Method & \multicolumn{3}{c}{Online Methods}\\
 & & Reformulation & Iterative Comp. & Incremental Comp.\\
 & & ({\sc clasp}) & ({\sc clasp}, perl) & ({\sc clasp-nk})\\
\hline
2 most similar ~ &- &6 min. 45 sec.      & 6 min. 53 sec.      & 7 min. 17 sec. \\
($k=1$)          &OM& 106 MB                &  73 MB             &  111 MB\\
                 &- & $k=1$              & $k=1$             & $k=1$\\
\hline
2 most diverse ~ &- &33 min. 28 sec.             & 11 min.       & 7 min. 40 sec.\\
($k=22$)         &OM& 213 MB                & 73 MB             & 112 MB\\
                 &- & $k=22$             &  $k=22$              & $k=21$\\
\hline
3 most similar ~ &- &7 min. 5 sec.      &  7 min. 3 sec.      & 7 min. 21 sec.\\
($k=1$)          &OM& 141 MB               &  73 MB               & 112 MB\\
                 &- & $k=1$              & $k=1$               & $k=2$\\
\hline
3 most diverse ~ &- &78 min 42 sec.             &  18 min. 49 sec.     & 12 min. 40 sec.\\
($k=22$)         &OM& 333 MB               &      73 MB          & 167 MB\\
                 &- & $k=22$             &        $k=21$       & $k=21$\\
\hline
6 most similar ~ &- &64 min. 42 sec.             & 7 min. 32 sec.       & 7 min. 18 sec.  \\
($k=1$)          &OM& 584 MB               &           73 MB    &  112 MB\\
                 &- & $k=1$              &      $k=1$          & $k=2$\\
\end{tabular}
\end{center}
\end{table}

\section{Related Work}
\label{sec_Related}

Finding similar/diverse solutions has been studied in other areas
such as propositional logic \cite{Distance-SAT-99}, constraint
programming \cite{heb07,heb05}, and
planning~\cite{SrivastavaNguyenEtAl-Domain-2007}. Let us briefly
describe related work in each area, and discuss the similarities
and the differences compared with our approach.

\paragraph{Related Work in Propositional Logic}
In \cite{Distance-SAT-99}, Bailleux and Marquis study the following problem

\begin{itemize}
\item[]
{\sc distance-sat}\\
Given a CNF formula $\Sigma$, a partial interpretation $PI$, and a nonnegative integer $d$,
decide whether there is a model $I$ of $\Sigma$ such that $I$ disagrees with $PI$ on at most
$d$ atoms.
\end{itemize}

\noindent This problem is similar to {\sc $k$-close solution} in that it asks for
a $k$-close solution. On the other hand, it asks for a solution $k$-close to a
partial solution, whereas {\sc $k$-close solution} asks for a solution that is $k$-close
to a set of previously computed solutions. Also, {\sc distance-sat} considers a distance measure
(i.e., partial Hamming distance) computable in polynomial time; whereas {\sc $k$-close solution}
considers any distance measure such that deciding whether the distance of a set of solutions
is less than a given $k$ is in \NP. Despite these differences, with $S$ containing a single solution and $\Delta$
being a partial Hamming distance, {\sc $k$-close solution} becomes essentially the same as
{\sc distance-sat}.

As for the computational complexity analysis, Proposition~1 of \cite{Distance-SAT-99}
shows that in the general case {\sc distance-sat} is \NP-complete. However, determining
whether $\Sigma$ has a model that disagrees with a complete interpretation $I$
on at most $d$ variables, where $d$ is a constant, is in \Pol.

To solve {\sc distance-sat}, the authors propose two algorithms,
$\ii{DP}_{\is{distance}}$ and $\ii{DP}_{\is{distance+lasso}}$.
Our modification of {\sc clasp}'s algorithm is similar to the first
algorithm in that both algorithms check whether a partial interpretation
computed in the DPLL-like search obeys the given distance constraints.
On the other hand, unlike $\ii{DP}_{\is{distance}}$, {\sc clasp} also uses
conflict-driven learning: when it learns a conflicting set of
literals, it will never try to select them  in the later stages of
the search. $\ii{DP}_{\is{distance+lasso}}$ offers manipulations
while selecting a new variable: it creates a set of candidate
variables with respect to the distance function, computes weights of
these variables relative to the distance function, and selects  one
with the maximum weight. On the other hand, in {\sc select}, {\sc
clasp} creates a set of candidate variables, and selects one of the
candidates to continue the search. Using the idea of
$DP_{\is{distance+lasso}}$, we can modify {\sc clasp} further to
manipulate the selection of variables with respect to the distance
function. However, in the phylogeny reconstruction problem, since
the domain of the distance function consists of the edge atoms which
are far outnumbered by many auxiliary atoms, in {\sc select} the set
of candidate variables generally consists of only auxiliary
variables; due to these cases, the manipulation of the selection of
variables is not expected to improve the computational efficiency.

\paragraph{Related Work in Constraint Programming}
In \cite{heb07,heb05}, the authors study various computational
problems related to finding similar/diverse solutions. The main
decision problems studied in \cite{heb05} are the following:

\begin{itemize}
\item[]
{\sc $d$Distant$k$Set} (resp. {\sc $d$Close$k$Set})\\
Given a polynomial-time decidable and
polynomially balanced relation $R$ over strings, a symmetric, reflexive, total
and polynomially bounded distance function $\delta$ between strings,
and some string $x$, decide whether there is a set $S$ of $k$ strings
(i.e., $S \subset \{y\ |\ (x,y)\in R\}$) such that $\ii{min}_{y,z\in S} \delta(y,z) \geq d$
(resp. $\ii{max}_{y,z\in S}\delta(y,z) \leq d$).
\end{itemize}

\noindent
which are similar to  {\sc $d$-distant set} (resp. {\sc $d$-close set}):
\cite{heb05} asks for a set of $k$ solutions $d$-distant/close to one solution,
whereas {\sc $d$-distant/close set} asks for a set of solutions that is $d$-close/distant
to a set of previously computed solutions. Also, the distance measure considered in
{\sc $d$Distant$k$Set} (resp. {\sc $d$Close$k$Set}) is computable in polynomial time;
in {\sc $d$-distant set} (resp. {\sc $d$-close set})
deciding whether the distance of a set of solutions is less than a given $d$ is
assumed to be in \NP.

The main decision problems studied in \cite{heb07} are the following:

\begin{itemize}
\item[]
{\sc $d$Distant} (resp. {\sc $d$Close})\\
Given a constraint satisfaction problem $P$ with variables ranging over finite domains,
a symmetric, reflexive, total
and polynomially bounded distance function $\delta$ between partial instantiations of variables,
and some partial instantiation $p$ of variables of $P$, decide whether there is a solution $s$ of $P$
such that $\delta(p,s) \ge d$
($\delta(p,s) \geq d$).
\end{itemize}

\noindent which are similar to {\sc $d$-distant solution} and {\sc $d$-close solution}.
On the other hand, \cite{heb07} asks for a solution $d$-close to a
partial solution rather than a set of solutions. Also, the distance measure considered in
these problems is computable in polynomial time. However, with $S$ containing a single
solution and $\Delta$ being computable in polynomial time, {\sc $d$-distant solution}
(resp. {\sc $d$-close solution}) becomes essentially the same as {\sc $d$Distant}
(resp. {\sc $d$Close}).

The authors also study some optimization problems related to these problems,
similar to the ones that we study above.

As for the computational complexity analysis of these problems, the authors find out that
they are all \NP-complete; these results comply with the ones presented in
Section~\ref{sec_CompResu} subject to conditions under which the problems of \cite{heb05,heb07}
above are equivalent to the problems we study in this paper.

Considering partial Hamming distance as in \cite{Distance-SAT-99},
Hebrard et al.~present an
offline method (similar to our method) that applies clustering
methods, and two online methods: one based on reformulation (similar
to Online Method 1), the other based on a greedy algorithm (similar
to Online Method 2) that iteratively computes a solution that
maximizes similarity to previous solutions. The computation of a
$k$-close solution is due to a Branch \& Bound algorithm (similar to
the idea behind Online Method 3) that propagates some
similarity/diversity constraints specific to the given distance
function. Our offline/online methods are inspired by these methods
of \cite{heb07,heb05}.

We note that partial Hamming distance not unrelated to the ones
introduced for comparing phylogenies in Section
\ref{ssec_DistMeasPhyl}; one can polynomially reduce nodal distance
to partial Hamming distance, and vice versa also partial Hamming
distance to nodal distance of trees (allowing auxiliary atoms in the
LP encoding).

\paragraph{Related Work in Planning}
In \cite{SrivastavaNguyenEtAl-Domain-2007}, the authors study the following decision problem:

\begin{itemize}
\item[]
{\sc $d$Distant$k$Set} (resp. {\sc $d$Close$k$Set})\\
Find a set $S$ of $k$ plans for a planning problem $PP$,
such that $\ii{min}_{y,z\in S} \delta(y,z) \geq d$
(resp. $\ii{max}_{y,z\in S}\delta(y,z) \leq d$).
\end{itemize}

\noindent The authors study these problems considering domain-independent
distance measures  computable in polynomial time (such as Hamming distance or set difference).
They present a method (similar to our Online Method 1), where they add global constraints
to the underlying constraint satisfaction solver of the GP-CSP
planner~\cite{DoKambhampati-Planning-2001}. As another method they
present a greedy approach (similar to our Online Method 2), where
they add global constraints which forces solver to compute
$k$-diverse solutions in each iteration until it computes $n$
solutions. They also present a method (similar to our Online Method
3) which modifies an existing planner's~\cite{GereviniSaettiEtAl-Planning-2003}
heuristic function and computes $n$ $k$-similar solutions in the search level.

\paragraph{Advantages of using ASP-Based Methods/Tools}
Our ASP-based methods for computing similar/diverse or close/distant
solutions to a given problem have three main advantages compared to
other approaches:
\begin{itemize}
\item
they are not restricted to some domain-independent distance
function, like (partial) Hamming distance considered in all the
methods/tools above;
\item
depending on the particular ASP-based method, we can represent
domain-independent or domain-specific distance functions in ASP or
implement them in C++;
\item
we can use the definitions of distance functions modularly, without
modifying the main problem description or without modifying the
search algorithm or the implementation of the solver.
\end{itemize}
Thus, our ASP-based methods/tools for computing similar/diverse or
close/distant solutions are applicable to various problems with
different (often domain-specific) distance measures.

In that sense, a user may prefer to use our ASP-based methods/tools
for computing similar/diverse or close/distant solutions to a given
problem, compared to the SAT-based methods/tools and the CP-based
methods/tools, if the user considers a domain-specific distance
function but does not want to modify the CP/SAT solvers to be able to
compute similar/diverse or distant/close solutions.

Also, our ASP-based methods/tools may be preferred when it is easier
to represent the main problem in ASP, due to advantages inherited
from the expressive representation language of ASP, such as being
able to define the transitive closure. Some sample applications
include phylogenetic network reconstruction \cite{erd06a} and wire
routing \cite{coban08,erd04}.

ASP-based methods for computing similar/diverse or close/distant
solutions to a given problem are probably most useful for existing
well-studied ASP applications, such as phylogeny reconstruction~\cite{bro05,bro07}
or product configuration~\cite{SoiNie-Developing-98}, to be used
with domain-specific measures.

\section{Conclusion}
\label{sec_Disc}

We have studied two kinds of computational problems related to
finding similar/diverse solutions of a given problem, in the context
of ASP: one problem asks for a set of $n$ $k$-similar (resp.
$k$-diverse) solutions; the other asks, given a set of solutions,
for a $k$-close ($k$-distant) solution $s$.  We have analyzed the
computational complexity of these problems, and introduced
offline/online methods to solve them. We have applied these
offline/online methods to the phylogeny reconstruction problem, and
observed their effectiveness in computing similar/diverse
phylogenies for Indo-European languages. Similarly we have applied
these methods to planning problems, and observed their effectiveness
in computing in particular diverse plans in Blocks World.
Finally, we have compared our
work with related approaches based on other formalisms.

There are many appealing ASP applications for which finding similar/diverse
solutions could be useful. In this sense, our methods and implementation
(i.e., \claspnk) can be useful for ASP.
On the other hand, no existing phylogenetic system can
compute similar/diverse phylogenies. In this sense, our distance
functions, methods, and tools
can be useful for phylogenetics. Similarly, no planner can compute
similar/diverse plans with respect to a domain-specific measure,
our methods and tools can be useful for planning.
In general,  the ASP-based methods/tools can be useful for finding
similar/diverse or close/distant solutions to a problem in two cases:
representing the problem in ASP is easier (e.g., if the problem
involves recursive definitions like transitive closure), or
the distance measure is different from the Hamming distance
(implemented in the SAT/CP-based tools).

We are also interested in combinations of the problems studied above
(for instance, finding a phylogeny that is the most distant from a
given set of phylogenies and that is the closest to another given
set of phylogenies), and application of our methods to other
problems. The study of these problems is left as a future work.

\subsection*{Acknowledgments}

We are grateful to the reviewers
of the paper as well as the reviewers of the preliminary conference
version for their comments and constructive suggestions for
improvement, in particular regarding the computation of nodal
distances.



\newpage
\appendix

\section{Proofs of Theorems}
\label{sec:proofs}

\begin{proof}[Proof of Theorem~\ref{theo:n-k-sim}]
Membership: Consider a non-deterministic Turing machine $M$,
operating on input $\cal P$, $M^\leq_\Delta$ (\resp $M^\geq_\Delta$),
$n$, and $k$, which guesses
$S$ as a set $\{s_1,\ldots,s_n\}$ of $n$ interpretations over the
alphabet of $\cal P$, together with a potential witness $w$ for a
computation of $M^\leq_\Delta$  (\resp $M^\geq_\Delta$) 
of length polynomial in $n$. After
that, for $1\leq i\leq n$, $M$ checks whether
$s_i$ is an answer set of $\cal P$ and whether all $s_i$ represent distinct
solutions of the problem. It rejects if any of these tests
does not succeed. Otherwise, $M$ proceeds by verifying that $w$ is a
witness of $M^\leq_\Delta$  (\resp $M^\geq_\Delta$) given $S$ and $k$
as its input.
If so, then $M$ accepts, otherwise it
rejects. Since $n$ is polynomial in the size of the input to $M$,
this also holds for the guess of $M$. Moreover, the subsequent
computation of $M$, i.e., the tests carried out, can be accomplished
in time polynomial in $n$. Therefore, $M$ is a non-deterministic
Turing machine which decides {\sc $n$ $k$-similar solutions}
(respectively {\sc $n$ $k$-diverse solutions}) in polynomial time,
which proves \NP-membership for these problems.

Hardness: Let $\phi=\bigwedge_{1\leq i\leq l} c_i$ be a Boolean
formula in conjunctive normal form (CNF) over variables $B=\{b_1,
\ldots, b_m\}$, i.e., each $c_i$ is a clause over variables from
$B$. By $\bar{x}$ we denote the complement of a literal $x$, i.e.,
$\bar{x} = \neg b$ if $x=b$, and $\bar{x} = b$ if $x=\neg b$. This
notation is extended to clauses in the obvious way: $\bar{c} =
\bar{x}_1 \wedge \ldots \wedge \bar{x}_{l_c}$ for a clause $c = x_1
\vee \ldots \vee x_{l_c}$.

Consider the normal logic program $\mathcal{P} = \{b_i \lar \no
nb_i;\ nb_i \lar \no b_i\mid 1\leq i\leq m\}\cup \{\lar
\bar{c}'_i\mid 1\leq i\leq l\}$, where $\bar{c}'$ denotes the
conjunction obtained from $\bar{c}$ by replacing negative literals
$\neg x$ in $\bar{c}$ by $nx$ (and using `,' instead of
`$\wedge$'). It is easily verified (and well-known) that $\cal P$
has an answer set iff $\phi$ is satisfiable (and that every answer
set of $\cal P$ is in 1-to-1 correspondence with a satisfying
assignment of $\phi$ in the obvious way).

Given $\cal P$, consider the {\sc $n$ $k$-similar solutions}
(respectively {\sc $n$ $k$-diverse solutions}) problem, where $n=1$,
$k=0$, and for any set $S$ of answer sets of $\cal P$, the distance
measure $\Delta$ is defined by $\Delta(S)=0$. Note that $\Delta$ is
normal and computable in constant time.
Then, there exists a solution to the problem iff there exists a set
$S$ of answer sets of  $\cal P$ such that $|S|=1$, i.e., $\cal P$
has an answer set. Since $\cal P$ has an answer set iff $\phi$ is
satisfiable, this proves \NP-hardness of the {\sc $n$ $k$-similar
solutions} (respectively {\sc $n$ $k$-diverse solutions}) problem.
Note that this argument holds for any normal $\Delta$.
\end{proof}

\begin{proof}[Proof of Theorem~\ref{theo:k-close}]
Membership: Consider a non-deterministic Turing machine $M$,
operating on input $\cal P$, $M^\leq_\Delta$ (\resp $M^\geq_\Delta$),
a set $S$ of solutions given by answer sets of $\cal P$, and $k$. It guesses an
interpretation $s$ over the alphabet of $\cal P$ (which is polynomial in the
size of $\cal P$), together with a potential witness $w$ for a computation of
$M^\leq_\Delta$ (\resp $M^\geq_\Delta$) 
of length polynomial in $|S|+|s|+\log k$. After
that, $M$ checks whether $s$ is an answer set of $\cal P$
and whether it represents a solution different from all solutions in $S$.
It rejects if any of these tests does not succeed. Otherwise, $M$ proceeds by
verifying that $w$ is a witness of $M^\leq_\Delta$ (\resp $M^\geq_\Delta$)
on input $S\cup\{s\}$ and $k$.
If so, then $M$ accepts, otherwise it rejects. The guess of $M$ is polynomial
in its input and the subsequent computation of $M$, i.e., the tests
carried out, can be accomplished in polynomial time. Therefore, $M$
is a non-deterministic Turing machine which decides {\sc $k$-close
solution} (respectively {\sc $k$-distant solution}) in polynomial
time, which proves \NP-membership for these problems.

Hardness: Consider the normal logic program given in the proof of
Theorem~\ref{theo:n-k-sim}, and the {\sc $k$-close solution}
problem, where $S=\emptyset$, $k=0$, and for any set $S'$ of answer
sets of $\cal P$, the distance measure $\Delta$ is defined by
$\Delta(S')=0$. Note that $\Delta$ is normal and computable in constant time.
Then, there exists a solution to the problem iff there exists a set $S'$ of
answer sets of $\cal P$ such that $S'\neq\emptyset$, i.e., $\cal P$ has an answer
set, which proves \NP-hardness of the {\sc $k$-close solution}
problem. Similarly, the {\sc $k$-distant solution} problem, where
$S=\emptyset$, $k=0$, and $\Delta(S')=0$, has a solution iff $\cal P$
has an answer set. Moreover, the above arguments hold for any normal $\Delta$.
This proves the claim.
\end{proof}

\begin{proof}[Proof of Theorem~\ref{theo:max-k-sim}]
Membership:
The problem of computing the cardinality of a maximal solution $S$ of
size at most $n$ is an optimization problem for a problem in \NP\ such
that the optimal value can be represented
using $\log n$ bits. Let
$M_\mathit{opt}$ be an oracle for this problem, and consider a
non-deterministic Turing machine $M'$, with output tape operating on
input $\cal P$, $M^\leq_\Delta$ (\resp $M^\geq_\Delta$), and $k$.
Initially, $M'$ calls $M_\mathit{opt}$ with $\cal P$, $M^\leq_\Delta$
(\resp $M^\geq_\Delta$), and $k$ as input to compute the maximum
cardinality $c$ of a set of solutions $S$ such that $|S|\leq n$. Then,
$M'$ proceeds like the nondeterministic Turing machine $M$ in the proof
of Theorem~\ref{theo:n-k-sim} using $n=c$, additionally writing the
guessed solution $S$ to its output tape. Since the latter is
accomplished in time polynomial in $c$, $M'$ is a non-deterministic
Turing machine with output tape that consults an oracle once for
computing the optimal value of an optimization problem solvable in
\NP. Thus, $M'$ is in \FNPLog\ and decides {\sc maximal $n$ $k$-similar
solutions} (respectively {\sc maximal $n$ $k$-diverse solutions}).

Hardness: We reduce $X$-MinModel to the problem of computing {\sc
maximal $n$ $k$-similar solutions}. $X$-MinModel is the following
\FNPLog-complete problem: Given a Boolean formula in CNF as in the
proof of Theorem~\ref{theo:n-k-sim}, and a subset $X$ of $B$,
compute an $X$-minimal model of $\phi$, i.e., a satisfying
assignment for $\phi$, which is subset minimal among all satisfying
assignments for $\phi$ on the variables from $X$ which are assigned
true. We identify a truth assignment with the set of Boolean
variables that are assigned true, and for a truth assignment $s$, we
use $s|_X$ to denote its restriction to variables from $X$.

Consider the normal logic program given in the proof of
Theorem~\ref{theo:n-k-sim}, and the {\sc maximal $n$ $k$-similar
solutions} problem, where $n=|X|$, $k=0$, and $\Delta$ is defined
as follows.
For a given set $S$ of answer sets of $\cal P$, such that $|S|>0$,
we set $\Delta(S)=0$ if for every pair of answer sets $s_1,s_2$ in
$S$, either $s_1|_X \subset s_2|_X$, or $s_2|_X \subset s_1|_X$.
Otherwise (and if $S=\emptyset$), $\Delta(S)=1$. Note that $\Delta$
is computable in polynomial time, performing less than $2n^2$ checks
for proper containment, where $|S|=n$. Observe also that the answer
sets in a set $S$ such that $\Delta(S)=0$, can be strictly ordered
wrt.~subset inclusion on their restrictions to $X$, and that $|X|$
is the maximum cardinality for such a set of answer sets. Given $S$
such that $\Delta(S)=0$, let $s_1$ denote the minimal answer set in
$S$ wrt.~subset inclusion on the restriction to $X$. The following
is trivial: the {\sc maximal $n$ $k$-similar solutions} problem above
has a solution iff $\phi$ is satisfiable. By the problem definition, it
also holds for every solution $S$ of the problem that $\Delta(S)=0$.

We show that if $S$ is a solution of the {\sc maximal $n$ $k$-similar
solutions} problem given above, then $s_1$ is an $X$-minimal model
of $\phi$. Towards a contradiction assume that there exists a
satisfying assignment $s'$ for $\phi$, such that $s'|_X \subset
s_1|_X$. Consider $s_0=s'\cup \{nb\mid b\in B, b\not\in s'\}$. Since
$s'$ satisfies $\phi$, it holds that $s_0$ is an answer set of $\cal
P$. Moreover $s_0\not\in S$, since $s_0|_X \subset s_1|_X$ and $s_1$
is the minimal answer set in $S$ wrt.~subset inclusion on the
restriction to $X$. As a consequence, $S\cup\{s_0\}\supset S$ and
$\Delta(S\cup\{s_0\})=0$. However, since the latter also implies
$|S\cup\{s_0\}|\leq n$, this contradicts our assumption
that $S$ is a solution of the {\sc maximal $n$ $k$-similar solutions}
problem above. We have thus shown that no satisfying assignment $s'$
for $\phi$ exists, such that $s'|_X \subset s_1|_X$, i.e., that
$s_1$ is an $X$-minimal model of $\phi$. This completes the
reduction of $X$-MinModel to the problem of computing {\sc maximal
$n$ $k$-similar solutions}, proving \FNPLog-hardness.

For a reduction of $X$-MinModel to the problem of computing {\sc
maximal $n$ $k$-diverse solutions}, we simply swap the values of
$\Delta$ and define: $\Delta(S)=1$ if  $|S|>0$ and for every pair of
answer sets $s_1,s_2$ in $S$, either $s_1|_X \subset s_2|_X$, or
$s_2|_X \subset s_1|_X$. Otherwise (and if $S=\emptyset$),
$\Delta(S)=0$. \FNPLog-hardness follows by analogous arguments.
\end{proof}

\begin{proof}[Proof of Theorem~\ref{theo:n-most-sim}]
Membership: Consider a deterministic Turing machine $M'$, with
output tape and an oracle for \NP-problems, which operates on input
$\cal P$, $M^\leq_\Delta$ (\resp $M^\geq_\Delta$), and $n$. Initially,
$M'$ prepares an integer $k_1$ of $n$ bits with the
less significant
half of bits set to $1$ and the
remaining bits set to $0$. Then, $M'$ successively uses its oracle
operating as the nondeterministic Turing machine $M$ in the proof of
Theorem~\ref{theo:n-k-sim}, starting with input $\cal P$,
$M^\leq_\Delta$ (\resp $M^\geq_\Delta$), $n$, and $k_1$, performing a
binary search for an optimal $k$. After that, $M'$ once more uses its
oracle like the nondeterministic Turing machine $M$ in the proof of
Theorem~\ref{theo:n-k-sim}, but additionally copying the solution
$S$ guessed by the oracle to its output tape. Since the latter is
accomplished in time polynomial in $n$, and since a polynomial
number of calls to the oracle is sufficient to complete the binary
search, $M'$ is in \FPNP and decides {\sc $n$ most similar
solutions} (respectively {\sc $n$  most diverse solutions}).

If the value of $\Delta(S)$ is polynomial in the size of $S$, then
the problem of computing the maximal value of $\Delta(S)$
over all solutions $S$ is an optimization problem for a problem in \NP\ such
that the optimal value can be represented using $\log n$ bits. Let
$M_\mathit{opt}$ be an oracle for this problem, and consider a
non-deterministic Turing machine $M''$ with output tape operating on
input $\cal P$, $M^\leq_\Delta$ (\resp $M^\geq_\Delta$), and $n$. Initially,
$M''$ calls $M_\mathit{opt}$ with $\cal P$, $M^\leq_\Delta$
(\resp $M^\geq_\Delta$), and $n$ as input to compute the value $k$ for
$\Delta(S)$ of an optimal solution $S$.
Then, $M''$ proceeds like the nondeterministic Turing machine $M$ in
the proof of Theorem~\ref{theo:n-k-sim}, additionally writing the
guessed solution $S$ to its output tape. Since the latter is
accomplished in time polynomial in $n$, $M''$ is in \FNPLog\ and
decides  {\sc $n$ most similar solutions} (respectively {\sc $n$
most diverse solutions}).

Hardness: We reduce the Traveling Salesman Problem (as, e.g.,
in~\cite{papa-94}) to the problem of computing {\sc $n$ most similar
solutions}. Consider $m$ cities $1,\ldots, m$, and a non-negative
integer distance $d_{i,j}$ between any two cities $i$ and $j$. The
task is to compute a tour visiting all cities once (i.e., a Hamilton
Cycle) of shortest length.

For a reduction, consider $\mathcal{P} \ = \ \{\ p_{i,j} \lar \no
np_{i,j};\ \ np_{i,j} \lar \no p_{i,j};\
\ r_j \lar p_{i,j};\ $\\
$\lar \no r_j \mid i\neq j\} \cup \{\lar p_{i,j}, p_{k,j};\ \lar
p_{i,j}, p_{i,k} \mid i\neq j,  i\neq k, j\neq k \}$,  where indices
$i$, $j$, and $k$ range over $1,\ldots, m$. Every answer set $s$ of
$\cal{P}$ uniquely corresponds to a Hamilton Cycle encoded by the
atoms $p_{i,j}$ in $s$, and every permutation of the cities gives
rise to exactly one answer set of $\cal P$. This can easily be
verified observing that the first two rules encode a
nondeterministic guess of a set of atoms $p_{i,j}$. The third and
fourth rule are satisfied iff `every city is reached', i.e., if for
every index $j$ there exists an index $i$, such that $p_{i,j}$ is
true. The last two rules are satisfied iff every city `is reached
from at most one different city' and `reaches at most one different
city', i.e., iff for different indices $i$, $j$, and $k$, $p_{i,j}$
and $p_{k,j}$ cannot both be true, as well as
 $p_{i,j}$ and $p_{i,k}$ cannot both be true.

Given this program, consider the {\sc $n$ most similar solutions}
problem, where $n=1$, and for any set $S$ of answer sets of $\cal
P$, the distance measure $\Delta$ is defined by
$\Delta(S)=\Sigma_{s\in S} \Sigma_{p_{i,j}\in s}\, d_{i,j}$. Note
that $\Delta$ is monotonic and thus computable in polynomial time in
the size of $S$.
Moreover, $S$ is a solution to this problem iff, by its definition,
$S$ contains exactly one answer set $s$ of $\cal P$, and iff
$\Delta(S)$ is minimal among all sets of answer sets of $\cal P$,
thus in particular among elementary such sets. By the definition of
$\Delta$, this implies that $S=\{s\}$ is a solution iff $s$ encodes
a Hamilton Cycle of minimal cost. This proves \FPNP-hardness for the
{\sc $n$ most similar solutions} problem in general.

For a reduction of TSP to the problem of computing {\sc $n$ most
diverse solutions}, consider $\Delta'(S) = m\times\mathit{max}_d -
\Delta(S)$, where $\mathit{max}_d$ is the maximum distance $d_{i,j}$
given. Also $\Delta'$ is monotonic and computable in polynomial time,
and by analogous arguments \FPNP-hardness follows for the {\sc $n$ most
diverse solutions} problem in the general case.

If the value of $\Delta(S)$ is polynomial in the size of $S$, then
\FNPLog-hardness is obtained by a reduction of $X$-MinMod: 
Let $\cal P$ be the normal logic program in the proof of
Theorem~\ref{theo:n-k-sim}, and consider the {\sc $n$ most similar
solutions} problem, where $n=1$, and $\Delta(S)$ is given by the
minimal (respectively maximal) partial Hamming distance on $X$
between an answer set $s\in S$ and $\emptyset$ (respectively $X$).
It is a straightforward consequence of the definition of $\Delta$,
that if $S=\{s\}$ is a solution to this {\sc $n$ most similar
solutions} problem (respectively to this {\sc $n$ most diverse
solutions} problem), then $s$ is an $X$-minimal model of $\phi$
(cf.~also the proof of Theorem~\ref{theo:max-k-sim}).
\end{proof}

\begin{proof}[Proof of Theorem~\ref{theo:most-close}]
Membership: Consider a deterministic Turing machine $M'$, with
output tape and an oracle for \NP-problems, which operates on input
$\cal P$, $M^\leq_\Delta$ (\resp $M^\geq_\Delta$), and $S$. Initially,
$M'$ prepares an integer $k_1$ of $n$ bits with the less significant
half of bits set to $1$ and the remaining bits set to $0$. Then, $M'$
successively uses its oracle operating as the nondeterministic Turing
machine $M$ in the proof of Theorem~\ref{theo:k-close}, starting with
input $\cal P$, $M^\leq_\Delta$ (\resp $M^\geq_\Delta$), $S$, and $k_1$,
performing a binary search for an optimal $k$. After that, $M'$ once more
uses its oracle like the nondeterministic Turing machine $M$ in the proof of
Theorem~\ref{theo:k-close}, but additionally copying the answer set
$s$ guessed by the oracle to its output tape. Since the latter is
accomplished in time polynomial in $n$, and since a polynomial
number of calls to the oracle is sufficient to complete the binary
search, $M'$ is in \FPNP and decides {\sc closest solution}
(respectively {\sc most distant solutions}).

If the value of $\Delta(S)$ is polynomial in the size of a set $S$ of
$n$ solutions, then
the problem of computing the maximal value of $\Delta(S\cup\{s\})$
for any solution $S\cup\{s\}$ is an optimization problem for a problem
in \NP\ such that the optimal value can be represented using
logarithmically many bits in the size of $S\cup\{s\}$. Let
$M_\mathit{opt}$ be an oracle for this problem, and consider a
non-deterministic Turing machine $M''$ with output tape operating on
input $\cal P$, $M^\leq_\Delta$ (\resp $M^\geq_\Delta$), and $S$.
Initially, $M''$ calls $M_\mathit{opt}$ with $\cal P$, $M^\leq_\Delta$
(\resp $M^\geq_\Delta$), and $S$ as input to compute the value $k$ for
$\Delta(S\cup\{s\})$ of an optimal solution $S\cup\{s\}$. Then, $M''$
proceeds like the nondeterministic Turing machine $M$ in the proof of
Theorem~\ref{theo:k-close}, additionally writing the guessed answer
set $s$ to its output tape. Since the latter is accomplished in time
polynomial in the input, $M''$ is in \FNPLog\ and decides {\sc
closest solution} (respectively {\sc most distant solutions}).

Hardness: The respective lower bounds are an immediate consequence
of the problem reductions in the proof of the previous
Theorem~\ref{theo:n-most-sim}. Just observe that for given $\cal P$
and $\Delta$,  the solutions of an {\sc $n$ most similar solutions}
problem with input $n=1$ coincide with the solutions of the {\sc
closest solution} problem with input $S=\emptyset$, (and the same
holds for the problem {\sc $n$ most diverse solutions} with input
$n=1$ and {\sc most distant solution} with input $S=\emptyset$). It
thus suffices to recall that the reductions mentioned above are
reductions to problems with input $n=1$.
\end{proof}

\begin{proof}[Proof of Theorem~\ref{theo:k-close-set}]
Membership: Consider a non-deterministic Turing machine $M$,
operating on input $\cal P$,  $M^\leq_\Delta$, $M^\geq_\Delta$,
a set $S$ of answer sets of $\cal P$, and $k$. Let $n$ be the size of its input.
First, $M$ guesses $S'$, such that $|S'|$ is polynomial in $n$, as a set
$\{s'_1,\ldots,s'_m\}$ of interpretations over the alphabet of $\cal
P$, two integers $k_1$ and $k_2$ in binary representation of size at
most polynomial in $n$, together with two potential witnesses $w_1$
and $w_3$ of $M^\leq_\Delta$ and of length polynomial in $|S|$ and $|S'|$,
respectively, as well as two potential witnesses $w_2$ and $w_4$ of
$M^\geq_\Delta$ of length polynomial in $|S|$ and $|S'|$, respectively.
After that, $M$ checks whether $S'$ is different from $S$, as well as whether
$s'_i$ is an answer set of $\cal P$, for $1\leq i\leq m$. It rejects if any
of these tests does not succeed.
Otherwise, $M$ proceeds by verifying that $w_1$ is a witness of
$M^\leq_\Delta$ on input $S$ and $k_1$, that $w_2$ is a witness of
$M^\geq_\Delta$ on input $S$ and $k_1$, as well as that
$w_3$ is a witness of $M^\leq_\Delta$ on input $S'$ and $k_2$, and
that $w_4$ is a witness of $M^\geq_\Delta$ on input $S'$ and $k_2$.
If either test fails $M$ rejects, otherwise it checks whether $|k_1-k_2| \leq k$
(respectively $|k_1-k_2| \geq k$), and if so accepts, otherwise it
rejects. Note that due to our assumptions that the size of $S'$ to
consider is polynomial in $n$, and that the value of $\Delta(S)$,
respectively $\Delta(S')$ is bounded by an exponential in the size
of $S$, respectively in the size of $S'$, the guess of $M$, which is
polynomial in $n$, is sufficient for deciding the problem. Moreover,
the subsequent computation of $M$, i.e., the tests carried out, can
be accomplished in polynomial time. Therefore, $M$ is a
non-deterministic Turing machine which decides {\sc $k$-close set}
(respectively {\sc $k$-distant set}) in polynomial time, which
proves \NP-membership for these problems.

Hardness: Consider the normal logic program given in the proof of
Theorem~\ref{theo:n-k-sim}, and the {\sc $k$-close set} problem,
where $S=\emptyset$, $k=0$, and for any set $S'$ of answer sets of
$\cal P$, the distance measure $\Delta$ is defined by $\Delta(S')=0$.
Note that $\Delta$ is normal and computable in constant time. Then,
there exists a solution to the problem iff there exists a set $S'$ of
answer sets of  $\cal P$ such that $S'\neq\emptyset$, i.e., $\cal P$
has an answer set, which proves \NP-hardness of the {\sc $k$-close
set} problem. Similarly, the {\sc $k$-distant set} problem, where
$S=\emptyset$, $k=0$, and $\Delta(S')=0$, has a solution
iff $\cal P$ has an answer set. Again the arguments hold for any
normal $\Delta$, which proves the claim.
\end{proof}

\begin{proof}[Proof of Proposition~\ref{pro_nodal}]
  In order to compute $D_n(P_1,P_2)$, we need to perform
  $|L| \choose 2$ nodal distance computations where $|L|$ is the number
  of leaves. The nodal distance $\ii{ND}_P(x,y)$  between leaves
  $x$ and $y$ in a phylogeny $P$ can be computed as
$$\ii{ND}_P(x,y) = \ii{depth}_P(x) + \ii{depth}_P(y) - 2
  \times \ii{depth}_P(\ii{lca}_P(x,y)),
$$
where $\ii{lca}_P(x,y)$ is the
  lowest common ancestor of $x$ and $y$ in $P$. Note that, if
  $\ii{depth}_P(v)$ for all vertices $v$ in $P$ is given (which is
  computable in $O(|L|)$ time, as $P$ is a binary tree), the computation
  of $\ii{ND}_P(x,y)$ takes constant time if
  $\ii{lca}_P(x,y)$ is known. Then, computing $\ii{ND}_P(x,y)$  for all
  leaves $x,y$ in $P$ is possible in $O(|L|^2)$ time.
  In a standard post-order traversal of $P$, a called node $v$
  always fulfills $v= \ii{lca}_P(x,y)$ for any vertices $x,y$ that occur
  in different subtrees rooted at children of $v$. Thus, if each call returns
  all leaves of $P$ reached from $v$ (which has overall cost $O(|L|)$),
  we can calculate in the traversal $\ii{ND}(x,y)$ for all leaves $x,y$
  of $P$
in the setting above.
  In total, the time to compute $\ii{ND}_{P_1}(x,y)$ and  $\ii{ND}_{P_2}(x,y)$, for all $x,y\in
  L$, is $2{\times}O(|L|) + O(|L|^2) = O(|L|^2)$.
  Therefore, in total $D_n(P_1,P_2)$ is computable in $O(|L|^2)$ time.
\end{proof}

\begin{proof}[Proof of Proposition~\ref{pro_compdesc}]
Let $v$ be the number of
vertices in one tree, then $v^2$ is an upper bound for the number of
the pairs that we can compare their descendants. Therefore, we have
at most $O(v^2)$ comparisons.

Since the number of descendants is bounded by $|L|$ (after obtaining
the descendants of each vertex by preprocessing in $O(v{\cdot}|L|)$
time), each comparison takes time $O(|L|)$.

Since $v = 2 \times |L| - 1$, $D_l(P_1,P_2)$ can be computed in $(2
\times |L| - 1)^2 \times |L|$ steps which is $O(|L|^3)$.
\end{proof}

\begin{proof}[Proof of Proposition~\ref{pro:lbn}]

Let $S_p'$ be a set of all completions of the partial phylogeny
$P_p$. For every $P \in S_p'$, we need to prove that
$$
{\cal LB}_n(P_p,P_c) \leq D_n(P,P_c)
$$
\noindent holds.

Let $P_l \in \arg\min_{P \in S_p'}(D_n(P,P_c))$ be a completion with
smallest distance. Then it will be
enough to prove that
$$
{\cal LB}_n(P_p,P_c) \leq D_n(P_l,P_c)
$$
\noindent holds. If we replace ${\cal LB}_n$ and $D_n$ with their
equivalents, the inequality will look like the following:
$$
\sum_{x,y \in L_p} |\ii{ND}_{P_c}(x,y) - \ii{ND}_{P_p}(x,y)| \leq
\sum_{x,y \in L} |\ii{ND}_{P_l}(x,y) - \ii{ND}_{P_c}(x,y)|
$$

\noindent We can break the right hand side summation into two for
$L_p$ and $L \backslash L_p$ as follows:
$$
\ba l
\sum_{x,y \in L_p} |\ii{ND}_{P_c}(x,y) - \ii{ND}_{P_p}(x,y)| \leq \\
\sum_{x,y \in L_p} |\ii{ND}_{P_l}(x,y) - \ii{ND}_{P_c}(x,y)| +
\sum_{(x,y) \in L^2 \backslash L_p^2} |\ii{ND}_{P_l}(x,y) -
\ii{ND}_{P_c}(x,y)|
\ea
$$

\noindent The distance between $x$ and $y$ is the same for $P_p$ and
$P_l$ where $x,y \in L_p$. Therefore, terms cancel each other and
we have the following:

$$
0 \leq \sum_{(x,y) \in L^2 \backslash L_p^2} |\ii{ND}_{P_l}(x,y) -
\ii{ND}_{P_c}(x,y)|
$$

\noindent Since the right hand side is a summation of absolute
values, the inequality holds which completes the proof.
\end{proof}

\begin{proof}[Proof of Proposition~\ref{pro:ubn}]

Let $S_p'$ be a set of all completions of the partial phylogeny
$P_p$. For every $P \in S_p'$, we need to prove that
$$
{\cal UB}_n(P_p,P_c) \geq D_n(P,P_c) .
$$

Let $P_u \in \arg\max_{p \in S_p'}(D_n(p,P_c))$ be a completion at
largest distance. Then it will be
enough to prove that
$$
{\cal UB}_n(P_p,P_c) \geq D_n(P_u,P_c) .
$$
If we replace ${\cal UB}_n$ and $D_n$ with their
definition, the inequality is
$$
\sum_{x,y \in L_p} |\ii{ND}_{P_c}(x,y) - \ii{ND}_{P_p}(x,y)| +
({l \choose 2} - {|L_p| \choose 2}) \times l \geq
\sum_{x,y \in L} |\ii{ND}_{P_l}(x,y) - \ii{ND}_{P_c}(x,y)|.
$$
We can break the right hand side summation into two for
$L_p$ and $L \backslash L_p$ as follows:
$$
\ba l
\sum_{x,y \in L_p} |\ii{ND}_{P_c}(x,y) - \ii{ND}_{P_p}(x,y)| +
({l \choose 2} - {|L_p| \choose 2}) \times l \geq \\
\sum_{x,y \in L_p} |\ii{ND}_{P_u}(x,y) - \ii{ND}_{P_c}(x,y)| +
\sum_{x,y \in L \backslash L_p} |\ii{ND}_{P_u}(x,y) - \ii{ND}_{P_c}(x,y)|
\ea
$$
\noindent The distance between $x$ and $y$ is same for $P_p$ and
$P_u$ where $x,y \in L_p$. Terms cancel each other:
$$
({l \choose 2} - {|L_p| \choose 2}) \times l \geq
\sum_{x,y \in L \backslash L_p}
|\ii{ND}_{P_u}(x,y) - \ii{ND}_{P_c}(x,y)| .
$$
The maximum nodal distance in a tree is equal to the
number of leaves; therefore, each term in the right hand side of the
inequality is at most $l$. Since, there are
 $({l \choose 2} - {|L_p| \choose 2})$ terms in
the right hand side summation, $({l \choose 2} - {|L_p| \choose 2}) \times l$
is greater than or equal to the summation.
\end{proof}

\begin{proof}[Proof of Proposition~\ref{planning_lb}]
Take any plan-completion $X$ of the partial plan $P_p$.
Consider two cases.

\paragraph{Case 1: $|X| \leq |P_c|$.}
Our goal is to prove that
$$
{\cal LB}_h(P_p,P_c) \leq D_h(X,P_c) .
$$
By the definition of $D_h$, the distance between $X$ and $P_c$ is:
$$
D_h(X,P_c) =
|\{i \ |\ \ii{act}_{X}(i) \neq \ii{act}_{P_c}(i), 1 \leq i \leq |X| \}|
+ |P_c| - |X| .
$$
Since $X$ is a plan-completion of $P_p$ and $|X| \leq |P_c|$,
$\dom \ii{act}_{P_p} \subseteq \dom \ii{act}_{P_c}$;
then, by the definition of ${\cal LB}_h$:
$$
{\cal LB}_h(P_p,P_c) =
|\{i \ |\ \ii{act}_{P_p}(i) \neq \ii{act}_{P_c}(i), i\in \dom \ii{act}_{P_p}\}| .
$$
Since $X$ is a plan-completion of $P_p$,
$$
\{i \ |\ \ii{act}_{P_p}(i) \neq \ii{act}_{P_c}(i), i\in \dom \ii{act}_{P_p}\}
\subseteq
\{i \ |\ \ii{act}_{X}(i) \neq \ii{act}_{P_c}(i), 1\leq i \leq |X|\} .
$$
Hence,
$$
\ba{c}
{\cal LB}_h(P_p,P_c)
\leq
|\{i \ |\ \ii{act}_{X}(i) \neq \ii{act}_{P_c}(i), 1\leq i \leq |X|\}| + |P_c| - |X|
=
D_h(X,P_c) .
\ea
$$

\paragraph{Case 2: $|X| > |P_c|$.}
Our goal is to prove that
$$
{\cal LB}_h(P_p,P_c) \leq D_h(P_c,X) .
$$
By the definition of $D_h$, the distance between $X$ and $P_c$ is:
$$
D_h(P_c,X) =
|\{i \ |\ \ii{act}_{X}(i) \neq \ii{act}_{P_c}(i), 1 \leq i \leq |P_c| \}|
+ |X| - |P_c| .
$$
By the definition of ${\cal LB}_h$:
$$
\ba{ll}
{\cal LB}_h(P_p,P_c) = &
|\{i \ |\ \ii{act}_{P_p}(i) \neq \ii{act}_{P_c}(i), i\in \dom \ii{act}_{P_p},
1\leq i \leq |P_c|\}| +\\
& |\{i \ | \ l \geq i > |P_c|, i\in \dom \ii{act}_{P_p}\}| .
\ea
$$
Since $X$ is a plan-completion of $P_p$, $\ii{act}_{X}$ extends $\ii{act}_{P_p}$, and then
$$
\ba{c}
\{i \ |\ \ii{act}_{P_p}(i) \neq \ii{act}_{P_c}(i), i\in \dom \ii{act}_{P_p}, 1\leq i \leq |P_c|\}\\
\subseteq
\{i \ |\ \ii{act}_{X}(i) \neq \ii{act}_{P_c}(i), 1\leq i \leq |X|\} .
\ea
$$
Since $|X| > |P_c|$,
$$
|X| - |P_c| > |\{i \ | \ |X| \geq i > |P_c|, i\in \dom \ii{act}_{P_p}\}| =
|\{i \ | \ l \geq i > |P_c|, i\in \dom \ii{act}_{P_p}\}| .
$$
Hence,
$$
{\cal LB}_h(P_p,P_c) \leq
|\{i \ |\ \ii{act}_{X}(i) \neq \ii{act}_{P_c}(i), 1\leq i \leq |X|\}|
+ |X| - |P_c| = D_h(P_c,X) .
$$
\end{proof}

\begin{proof}[Proof of Proposition~\ref{planning_ub}]
Take any plan-completion $X$ of partial plan $P_p$.
Consider two cases.

\paragraph{Case 1: $|X| \leq |P_c|$.}
Our goal is to prove that
$$
{\cal UB}_h(P_p,P_c) \geq D_h(X,P_c)
$$
where
$$
\ba{r@{~}c@{~}l}
{\cal UB}_h(P_p,P_c) &=&
l - |\{i \ |\ \ii{act}_{P_p}(i) = \ii{act}_{P_c}(i), i\in \dom \ii{act}_{P_p}\}|, \\
D_h(X,P_c) &=&
|\{i \ |\ \ii{act}_{X}(i) \neq \ii{act}_{P_c}(i), 1 \leq i \leq |X| \}| + |P_c| - |X| .
\ea
$$
Since $|X| \leq |P_c|$ and $X$ is a plan-completion of $P_p$, the set
$$
\{i \ |\ \ii{act}_{P_p}(i) = \ii{act}_{P_c}(i), i\in \dom \ii{act}_{P_p}\}
$$
does not intersect with the set
$$
Y= \{i \ |\ \ii{act}_{X}(i) \neq \ii{act}_{P_c}(i),  1 \leq i \leq |X| \}
\cup\
\{i |\ |X| < i \leq |P_c|\} .
$$
Then
$$
\{1,...,l\} \setminus \{i \ |\ \ii{act}_{P_p}(i) = \ii{act}_{P_c}(i), i\in \dom \ii{act}_{P_p}\}
$$
is a superset of $Y$. Therefore,
$$
\ba{r@{~}c@{~}l}
{\cal UB}_h(P_p,P_c) &=& l - |\{i \ |\ \ii{act}_{P_p}(i) = \ii{act}_{P_c}(i), i\in \dom \ii{act}_{P_p}\}|
\\
&\geq& |\{i \ |\ \ii{act}_{X}(i) \neq \ii{act}_{P_c}(i),  1 \leq i \leq |X| \}| + |P_c| - |X|\\
&=& D_h(X,P_c) .
\ea
$$

\paragraph{Case 2: $|X| > |P_c|$.}
Our goal is to prove that
$$
{\cal UB}_h(P_p,P_c) \geq D_h(P_c,X)
$$
where
$$
\ba{r@{~}c@{~}l}
{\cal UB}_h(P_p,P_c) &=&
l - |\{i \ |\ \ii{act}_{P_p}(i) = \ii{act}_{P_c}(i), i\in \dom \ii{act}_{P_p}, 1\leq i \leq |P_c|\}|, \\
D_h(P_c,X) &=&
|\{i \ |\ \ii{act}_{X}(i) \neq \ii{act}_{P_c}(i), 1 \leq i \leq |P_c| \}| + |X| - |P_c| .
\ea
$$
Since $|X| > |P_c|$ and $X$ is a plan-completion of $P_p$, the set
$$
\{i \ |\ \ii{act}_{P_p}(i) = \ii{act}_{P_c}(i), i\in \dom \ii{act}_{P_p}, 1\leq i \leq |P_c|\}
$$
does not intersect with the set
$$
Y= \{i \ |\ \ii{act}_{X}(i) \neq \ii{act}_{P_c}(i),  1 \leq i \leq |P_c| \} \cup
\{i |\ |P_c| < i \leq |X|\} .
$$
Then
$$
\{1,...,l\} \setminus
\{i \ |\ \ii{act}_{P_p}(i) = \ii{act}_{P_c}(i), i\in \dom \ii{act}_{P_p}, 1\leq i \leq |P_c|\}
$$
is a superset of $Y$. Therefore,
$$
\ba{r@{~}c@{~}l}
{\cal UB}_h(P_p,P_c) &=&
l - |\{i \ |\ \ii{act}_{P_p}(i) = \ii{act}_{P_c}(i), i\in \dom \ii{act}_{P_p}, 1\leq i \leq |P_c|\}|
\\
&\geq& |\{i \ |\ \ii{act}_{X}(i) \neq \ii{act}_{P_c}(i), 1 \leq i \leq |P_c| \}| + |X| - |P_c|\\
&=& D_h(P_c,X) .
\ea
$$
\end{proof}


\newpage
\section{ASP Formulations}
\label{sec:pgms}

\begin{figure}[!h]
{
\begin{verbatim}
c{clique(X) : vertex(X)}c.
:- clique(X), clique(Y), vertex(X), vertex(Y), X!=Y, not edge(X,Y),
   not edge(Y,X).
\end{verbatim}
}
\caption{ASP formulation of the $c$-clique problem (a clique of
size $c$).}
\label{fig:clique}
\end{figure}

\begin{figure}[!h]
{
\begin{verbatim}
solution(1..n).
c{clique(S,X) : vertex(X)}c :- solution(S).
:- clique(S,X), clique(S,Y), not edge(X,Y), not edge(Y,X), X!=Y.
different(S1,S2) :- clique(S1,X), clique(S2,Y), S1 != S2, X != Y.
:- not different(S1,S2), solution(S1;S2), S1!=S2.
\end{verbatim}
}
\caption{ASP formulation that computes $n$ distinct $c$-cliques.}
\label{fig:cliqueN}
\end{figure}

\begin{figure}[!h]
{
\begin{verbatim}
same(S1,S2,V) :-  clique(S1,V), clique(S2,V), solution(S1;S2),
  vertex(V), S1 < S2.
hammingDistance(S1,S2,c-H) :- H{same(S1,S2,V): vertex(V)}H,
  solution(S1;S2), maximumDistance(H), S1 < S2.
\end{verbatim}
}
\caption{ASP formulation of the Hamming distance between two cliques.}
\label{fig:hamming}
\end{figure}

\begin{figure}[!h]
{
\begin{verbatim}
:- solution(S1;S2), hammingDistance(S1,S2,H), H > k,
   maximumDistance(H).
\end{verbatim}
} \caption{A constraint that forces the distance among any two
solution is less than or equal to $k$.} \label{fig:cons}
\end{figure}

\begin{figure}[!h]
{
\begin{verbatim}
% generate a rooted binary tree
vertex(0..2*k). root(2*k).
internal(X) :- vertex(X), not leaf(X).

2 {edge(X,Y) : vertex(Y) : X > Y} 2 :- internal(X).

reachable(X,Y) :- edge(X,Y), vertex(X;Y), X > Y.
reachable(X,Y) :- edge(X,Z), reachable(Z,Y),
   X > Z, vertex(X;Y;Z).
:- vertex(Y), not reachable(X,Y), root(X), Y != X.
:- reachable(X,X), vertex(X).

maxY(X,Y) :- edge(X,Y), edge(X,Y1), Y > Y1,
   vertex(Y;Y1), internal(X).
:- maxY(X,Y), maxY(X1,Y1), Y > Y1, X < X1,
   vertex(Y;Y1), internal(X;X1).
\end{verbatim}
}
\caption{The phylogeny reconstruction program
of Brooks et. al.: Part 1.}
\label{fig:solve1}
\end{figure}

\begin{figure}
{
\begin{verbatim}
% ensure that the tree does not have more than x incompatible characters
g0(X,I,S) :- f(X,I,S), informative_character(I),
   essential_state(I,S).
g0(Y,I,S) :- g0(X,I,S), g0(X1,I,S), edge(Y,X), edge(Y,X1),
   X>X1, internal(Y), vertex(X;X1), informative_character(I),
   essential_state(I,S).

marked(X,I) :- g0(X,I,S), informative_character(I),
   vertex(X), essential_state(I,S).

g(X,I,S) :- g0(X,I,S), informative_character(I),
   vertex(X), essential_state(I,S).
{g(X,I,S): essential_state(I,S)} 1 :- internal(X),
   not marked(X,I), informative_character(I).

{root_is(X,I,S)} :- g(X,I,S), vertex(X),
   informative_character(I), essential_state(I,S).

:- root_is(X,I,S), root_is(Y,I,S),
   vertex(X;Y), X < Y, informative_character(I),
   essential_state(I,S).
:- root_is(X,I,S), g(Y,I,S), reachable(Y,X), Y > X,
   vertex(X;Y), informative_character(I),
   essential_state(I,S).

reachable_is(X,I,S) :- root_is(X,I,S),
   vertex(X), informative_character(I), essential_state(I,S).
reachable_is(X,I,S) :- g(X,I,S), reachable_is(Z,I,S),
   edge(Z,X), Z > X, vertex(X;Z), informative_character(I),
   essential_state(I,S).

incompatible(I) :- g(X,I,S), not reachable_is(X,I,S),
   vertex(X), informative_character(I), essential_state(I,S).
:- n+1 {incompatible(I) : informative_character(I)}.
\end{verbatim}
}
\caption{The phylogeny reconstruction program
of Brooks et. al.: Part 2.}
\label{fig:solve2}
\end{figure}

\begin{figure}[!h]
{
\begin{verbatim}
% generate n rooted trees
solution(1..n).
vertex(0..2*k). root(2*k).
internal(X) :- vertex(X), not leaf(X).

2 {edge(N,X,Y) : vertex(Y) : X > Y} 2 :- internal(X), solution(N).

reachable(N,X,Y) :- edge(N,X,Y), vertex(X;Y), X > Y, solution(N).
reachable(N,X,Y) :- edge(N,X,Z), reachable(N,Z,Y), solution(N),
   X > Z, vertex(X;Y;Z).
:- vertex(Y), not reachable(N,X,Y), root(X), Y != X, solution(N).
:- reachable(N,X,X), vertex(X), solution(N).

maxY(N,X,Y) :- edge(N,X,Y), edge(N,X,Y1), Y > Y1,
   vertex(Y;Y1), internal(X), solution(N).
:- maxY(N,X,Y), maxY(N,X1,Y1), Y > Y1, X < X1,
   vertex(Y;Y1), internal(X;X1), solution(N).
\end{verbatim}
}
\caption{A reformulation of the phylogeny reconstruction program
of Brooks et. al. (Figures~\ref{fig:solve1} and \ref{fig:solve2}), to find $n$ distinct phylogenies: Part 1.}
\label{fig:solveN1}

\end{figure}

\begin{figure}
{
\begin{verbatim}
% ensure that no tree has more than x incompatible characters
g0(N,X,I,S) :- f(X,I,S), informative_character(I),
   essential_state(I,S), solution(N).
g0(N,Y,I,S) :- g0(N,X,I,S), g0(N,X1,I,S), edge(N,Y,X), edge(N,Y,X1),
   X>X1, internal(Y), vertex(X;X1), informative_character(I),
   essential_state(I,S), solution(N).

marked(N,X,I) :- g0(N,X,I,S), informative_character(I),
   vertex(X), essential_state(I,S), solution(N).

g(N,X,I,S) :- g0(N,X,I,S), informative_character(I),
   vertex(X), essential_state(I,S), solution(N).
{g(N,X,I,S): essential_state(I,S)} 1 :- internal(X),
   not marked(N,X,I), informative_character(I), solution(N).

{root_is(N,X,I,S)} :- g(N,X,I,S), vertex(X),
   informative_character(I), essential_state(I,S), solution(N).
:- root_is(N,X,I,S), root_is(N,Y,I,S),
   vertex(X;Y), X < Y, informative_character(I),
   essential_state(I,S), solution(N).
:- root_is(N,X,I,S), g(N,Y,I,S), reachable(N,Y,X), Y > X,
   vertex(X;Y), informative_character(I), essential_state(I,S),
   solution(N).

reachable_is(N,X,I,S) :- root_is(N,X,I,S),
   vertex(X), informative_character(I), essential_state(I,S),
   solution(N).
reachable_is(N,X,I,S) :- g(N,X,I,S), reachable_is(N,Z,I,S),
   edge(N,Z,X), Z > X, vertex(X;Z), informative_character(I),
   essential_state(I,S), solution(N).

incompatible(N,I) :- g(N,X,I,S), not reachable_is(N,X,I,S),
   vertex(X), informative_character(I), essential_state(I,S),
   solution(N).
:- x+1 {incompatible(N,I) : informative_character(I)}, solution(N).
\end{verbatim}
}
\caption{A reformulation of the phylogeny reconstruction program
of Brooks et. al. (Figures~\ref{fig:solve1} and \ref{fig:solve2}), to find $n$ distinct phylogenies: Part 2.}
\label{fig:solveN2}
\end{figure}

\begin{figure}
{
\begin{verbatim}
% make sure that these n trees are distinct

different(S1,S2) :- edge(S1,X1,Y), edge(S2,X2,Y),
   vertex(X2;X1;Y), solution(S1;S2), S1 != S2, X1 != X2.
:- not different(S1,S2), solution(S1;S2), S1 != S2.
\end{verbatim}
}
\caption{A reformulation of the phylogeny reconstruction program
of Brooks et. al., to find $n$ distinct phylogenies: Part 3.}
\label{fig:solveN3}
\end{figure}

\begin{figure}
{\begin{verbatim}
dist(0..m).

% compute the nodal distances using distance_v.

% nodaldistance(S,X,Y,T): the nodal distance between X and Y
%    in the S'th tree is T.
nodaldistance(S,X,Y,T) :- tempnodaldistance(S,X,Y,T),
   not notminnodal(S,X,Y,T), solution(S), leaf(X;Y), dist(T).

% distance_v(S,X,Y,T): the distance between the vertex X and
%    its descendant Y is T in the S'th tree.
distance_v(S,X,Y,1) :- edge(S,X,Y), vertex(X;Y), solution(S).
distance_v(S,X,Z,T+1) :- distance_v(S,X,Y,T), edge(S,Y,Z),
   vertex(X;Y;Z), dist(T), solution(S).

% length of a path between vertices X and Y
tempnodaldistance(S,X,Y,T1+T2) :- distance_v(S,CA,X,T1),
   distance_v(S,CA,Y,T2), X<Y, dist(T1;T2), leaf(X;Y),
   vertex(CA), solution(S).

notminnodal(S,X,Y,T1) :- tempnodaldistance(S,X,Y,T1),
   tempnodaldistance(S,X,Y,T2), T2 < T1, leaf(X;Y),
   dist(T1;T2), solution(S).

% compute the differences of nodal distances of each pairs of
%    leaves in each pairs of trees.

diffnodal(P1,P2,X,Y,abs(D1-D2)) :- nodaldistance(P1,X,Y,D1),
   nodaldistance(P2,X,Y,D2), P2>P1, leaf(X;Y), dist(D1;D2),
   solution(P1;P2).

% compute the distance between each pairs of trees.

% distance_t(P1,P2,T) : the distance between (trees) P1 and P2 is T.
tempdistance(P1,P2,0,1,D) :- diffnodal(P1,P2,0,1,D),
   solution(P1;P2), dist(D).
tempdistance(P1,P2,L1,L2,D+K) :- tempdistance(P1,P2,L1,L2-1,D),
   diffnodal(P1,P2,L1,L2,K), L2-L1 > 1, solution(P1;P2),
   leaf(L1;L2), dist(D;K).
tempdistance(P1,P2,L1,L2,D+K) :- tempdistance(P1,P2,L1-1,k,D),
   diffnodal(P1,P2,L1,L2,K), L2 = L1+1, L2 > 1, solution(P1;P2),
   leaf(L1;L2), dist(D;K).

distance_t(P1,P2,T) :- tempdistance(P1,P2,k-1,k,T), dist(T),
   solution(P1;P2).
\end{verbatim}
} \caption{A formulation of the nodal distance function $D_n$ in ASP.}
\label{fig:distance1}
\end{figure}

\begin{figure} {
\begin{verbatim}
dist(0..m).

% at each solution N, define reachability of leaf Y from X
reachableN(N,X,Y) :- edge(N,X,Y), vertex(X), leaf(Y), X > Y,
   solution(N).
reachableN(N,X,Y) :- edge(N,X,Z), reachableN(N,Z,Y),
   solution(N), X > Z, vertex(X;Z), leaf(Y).

% at each solution S, assign depths to vertices Y
depth(S,2*k,0) :- solution(S).
depth(S,Y,T+1) :- depth(S,X,T), edge(S,X,Y),
   vertex(X;Y), depthRange(T), solution(S), T<r.

% vertices V1 and v2 have different descendants
diff(N1,V1,N2,V2) :-  solution(N1;N2), vertex(V1;V2), leaf(X),
   N1 < N2, reachableN(N1,V1,X), not reachableN(N2,V2,X).
diff(N1,V1,N2,V2) :-  solution(N1;N2), vertex(V1;V2), leaf(X),
   N1 < N2, not reachableN(N1,V1,X), reachableN(N2,V2,X).

% definition of the function f
fN(N1,V1,N2,V2,1) :- diff(N1,V1,N2,V2), solution(N1;N2),
   vertex(V1;V2), N1 < N2.
fN(N1,V1,N2,V2,0) :- not diff(N1,V1,N2,V2), solution(N1;N2),
   vertex(V1;V2), N1 < N2.

% definition of the function g
gN(0,N1,N2,0) :- solution(N1;N2), N1 < N2.
gN(D+1,N1,N2,D1) :- gN(D,N1,N2,X), solution(N1;N2), N1 < N2,
   depthRange(D;Y), dist(Z;D1;X), maxdepth2(N1,N2,Y), D<Y,
   depthV2(N1,N2,D+1,2*k,Z), w(D+1,M), D1=X+M*Z.
\end{verbatim}
} \caption{An ASP formulation of the descendant distance function $D_l$ for two phylogenies: Part 1} \label{fig:distance2-1}
\end{figure}

\begin{figure}
{
\begin{verbatim}
% depthV2 computes the summation of f(x,y) over all x,y at the same depth
samedepth(N1,V1,N2,V2,D) :- depth(N1,V1,D), depth(N2,V2,D),
        vertex(V1;V2), solution(N1;N2), N1 < N2,  depthRange(D).

depthV(N1,N2,D,W,0,Z) :- solution(N1;N2), N1 < N2,  depthRange(D),
        vertex(W), samedepth(N1,W,N2,0,D), fN(N1,W,N2,0,Z), dist(Z).
depthV(N1,N2,D,W,0,0) :- solution(N1;N2), N1 < N2,  depthRange(D),
        vertex(W), not samedepth(N1,W,N2,0,D).
depthV(N1,N2,D,W,X+1,Z+Z1) :- solution(N1;N2), N1 < N2,
        depthRange(D), vertex(W), depthV(N1,N2,D,W,X,Z),
        samedepth(N1,W,N2,X+1,D), fN(N1,W,N2,X+1,Z1), dist(Z;Z1),
        vertex(X), X<2*k.
depthV(N1,N2,D,W,X+1,Z) :- solution(N1;N2), N1 < N2,  depthRange(D),
        vertex(W), depthV(N1,N2,D,W,X,Z),
        not samedepth(N1,W,N2,X+1,D), dist(Z;Z1), vertex(X), X<2*k.

depthV2(N1,N2,D,0,Z) :- solution(N1;N2), N1 < N2,  depthRange(D),
        depthV(N1,N2,D,0,2*k,Z), dist(Z).
depthV2(N1,N2,D,X+1,Z+Z1) :- solution(N1;N2), N1 < N2,
        depthRange(D), vertex(X), depthV2(N1,N2,D,X,Z),
        depthV(N1,N2,D,X+1,2*k,Z1), dist(Z;Z1), X<2*k.

% definition of the distance function D_n for two phylogenies
depth2(N1,N2,X) :- depth(N1,Y1,X), depth(N2,Y2,X),
        vertex(Y1;Y2), depthRange(X), solution(N1;N2), N1 < N2.
maxdepth2(N1,N2,X) :- depth2(N1,N2,X), not depth2(N1,N2,X+1),
        depthRange(X), solution(N1;N2), N1 < N2.

distance_t(N1,N2,X) :- gN(D,N1,N2,X), solution(N1;N2), N1 < N2,
        dist(X), depthRange(D), maxdepth2(N1,N2,D).
\end{verbatim}
} \caption{An ASP formulation of the descendant distance function $D_l$ for two phylogenies: Part 2}
\label{fig:distance2-2}
\end{figure}

\begin{figure}
\begin{verbatim}
% distance of a set of phylogenies
notmaxdistance_t(P1,P2,T1) :- distance_t(P1,P2,T1), distance_t(P3,P4,T2),
   T1 < T2, solution(P1;P2;P3;P4), dist(T1;T2).
delta(T1) :- distance_t(P1,P2,T1), not notmaxdistance_t(P1,P2,T1),
   solution(P1;P2), dist(T1).

% constraints on the distance function, for similarity
:- delta(T), dist(T), T > k.
\end{verbatim}
\caption{An ASP formulation of the distance function $\Delta_D$ for a set of phylogenies,
and the constraints for $k$-similarity.} \label{fig:constraint}
\end{figure}

\begin{figure}
{
\begin{verbatim}
% effect of moving a block
on(B,L,T1) :- block(B), location(L), moveop(B,L,T), next(T,T1).

% a block can be moved only when it's clear
:- location(L), block(B), block(B1), time(T),
   moveop(B,L,T), on(B1,B,T).

% any two blocks cannot be on the same block at the same time
:- 2{on(B1,B,T):block(B1)}, time(T), block(B).

% wherever a block is, it's not anywhere else
non(B,L1,T) :- time(T), location(L1), location(L), block(B),
   on(B,L,T), not eq(L,L1).

% every block is supported by the table
supported(B,T) :- block(B), time(T), on(B,table,T).
supported(B,T) :- block(B), block(B1), time(T), on(B,B1,T),
   supported(B1,T), not eq(B,B1).
:- block(B), time(T), not supported(B,T).

% no concurrency
:- 2{moveop(B,L,T):block(B):location(L)},time(T).

% inertia
on(B,L,T1) :- location(L), block(B), on(B,L,T), not non(B,L,T1),
   next(T,T1).

% initial values and actions are exogenous
1{non(B,L,0),on(B,L,0)}1 :- block(B), location(L).
:- non(B,L,T), on(B,L,T), block(B), location(L), time(T).

{moveop(B,L,T)} :- block(B), location(L), time(T), T < lasttime.

% auxiliary predicates
time(0..lasttime).
next(T,T+1) :- time(T), lt(T,lasttime).

location(L) :- block(L).
location(table).

goal :- time(T), goal(T).
:- not goal.
\end{verbatim}
} \caption{Blocks world formulation.} \label{fig:bw}
\end{figure}

\begin{figure}
{\begin{verbatim}
solution(1..n).

% effect of moving a block
on(S,B,L,T1) :- block(B), location(L),
   moveop(S,B,L,T), next(T,T1), solution(S).

% a block can be moved only when it's clear
:- location(L), block(B), block(B1), time(T),
   moveop(S,B,L,T), on(S,B1,B,T), solution(S).

% any two blocks cannot be on the same block at the same time
:- 2{on(S,B1,B,T):block(B1)}, time(T), block(B), solution(S).

% wherever a block is, it's not anywhere else
non(S,B,L1,T) :- time(T), location(L1), location(L), block(B),
   on(S,B,L,T), not eq(L,L1), solution(S).

% every block is supported by the table
supported(S,B,T) :- block(B), time(T), on(S,B,table,T), solution(S).
supported(S,B,T) :- block(B), block(B1), time(T), on(S,B,B1,T),
   supported(S,B1,T), not eq(B,B1), solution(S).
:- block(B), time(T), not supported(S,B,T), solution(S).

% no concurrency
:- 2{moveop(S,B,L,T):block(B):location(L)},time(T), solution(S).

% inertia
on(S,B,L,T1) :- location(L), block(B), on(S,B,L,T),
   not non(S,B,L,T1), next(T,T1), solution(S).

% initial values and actions are exogenous
1{non(S,B,L,0),on(S,B,L,0)}1 :- block(B), location(L), solution(S).
:- non(S,B,L,T), on(S,B,L,T), block(B), location(L), time(T), solution(S).

{moveop(S,B,L,T)} :- block(B), location(L), time(T), T < lasttime, solution(S).

% auxiliary predicates
time(0..lasttime).
next(T,T+1) :- time(T), lt(T,lasttime).

location(L) :- block(L).
location(table).

goal(S) :- time(T), goal(S,T), solution(S).
:- not goal(S), solution(S).

% compute distinct plans
different(S1,S2) :- time(T), moveop(S1,X,Y,T), not moveop(S2,X,Y,T),
   solution(S1;S2), block(X), location(Y), S1 < S2.
different(S1,S2) :- time(T), not moveop(S1,X,Y,T), moveop(S2,X,Y,T),
   solution(S1;S2), block(X), location(Y), S1 < S2.
:- not different(S1,S2), solution(S1;S2), S1 < S2.
\end{verbatim}
}
\caption{A reformulation of the Blocks World program shown in Fig.~\ref{fig:bw}, to compute
$n$ distinct plans.}
\label{fig:bw-n}
\end{figure}

\begin{figure}
{\begin{verbatim}
% for every time step T, check that the T'th actions
% of Plans P1 and P2 are different:
different(P1,P2,T) :- moveop(P1,X,Y,T), not moveop(P2,X,Y,T),
   time(T), solution(P1;P2), block(X), location(Y), P1 < P2.
different(P1,P2,T) :- not moveop(P1,X,Y,T), moveop(P2,X,Y,T),
   time(T), solution(P1;P2), block(X), location(Y), P1 < P2.

% and define the hamming distance between two plans P1 and P2
% in terms of these differences:
hammingdistance(P1,P2,H) :- H{different(P1,P2,T): time(T)}H,
   solution(P1;P2), distRange(H), P1 < P2.
\end{verbatim}
}
\caption{An ASP formulation of the Hamming distance $D_h$ for two plans.}
\label{fig:distance-h2}
\end{figure}

\begin{figure}
{\begin{verbatim}
somedistance(H) :- hammingdistance(P1,P2,H),
  solution(P1;P2), distRange(H).
notmaxdistance(H1) :- somedistance(H1), somedistance(H2),
  H2 > H1, distRange(H1;H2).
totaldistance(H) :- not notmaxdistance(H),
  distRange(H), somedistance(H).
\end{verbatim}
}
\caption{An ASP formulation of the distance $\Delta_h$ for a set of plans.}
\label{fig:distance-plans}
\end{figure}

\end{document}